\documentclass{article}

\PassOptionsToPackage{numbers}{natbib}
\PassOptionsToPackage{noend}{algorithmic}

\usepackage{fullpage}
\usepackage{authblk}
\usepackage{natbib}
\usepackage{hyperref} 
\usepackage{url}

\usepackage{booktabs} 
\usepackage{xspace}
\usepackage{wrapfig}

\usepackage{nicefrac} 
\usepackage{amsfonts} 
\usepackage{amsmath}
\usepackage{amssymb}
\usepackage{amsthm}
\usepackage{mathtools}

\usepackage{enumerate}
\usepackage{thm-restate}

\usepackage{algorithm}
\usepackage{algorithmic}

\theoremstyle{plain}
\newtheorem{theorem}{Theorem}[section]

\newtheorem{lemma}[theorem]{Lemma}

\theoremstyle{definition}
\newtheorem{definition}[theorem]{Definition}

\theoremstyle{remark}

\title{Improved Algorithms for Overlapping and Robust Clustering of Edge-Colored Hypergraphs: An LP-Based Combinatorial Approach\footnote{
Supported by NCN grant number 2020/39/B/ST6/01641.
This work was partly supported by Institute of Information \& communications Technology Planning \& Evaluation (IITP) grant funded by the Korea government (MSIT) (No. RS-2021-II212068, Artificial Intelligence Innovation Hub).
This work was partly supported by an IITP grant funded by the Korean Government (MSIT) (No. RS-2020-II201361, Artificial Intelligence Graduate School Program (Yonsei University)).
This work was supported by the National Research Foundation of Korea(NRF) grant funded by the Korea government(MSIT) (RS-2025-00563707).
Part of this research was conducted while Y. Shin was at Yonsei University.
}
}

\author[1]{Changyeol Lee\thanks{Co-first authors.}}
\author[2]{Yongho Shin\protect\footnotemark[2]}
\author[1]{Hyung-Chan An\thanks{Corresponding author: \texttt{hyung-chan.an@yonsei.ac.kr}}}
\affil[1]{Department of Computer Science, Yonsei University, Seoul, South Korea}
\affil[2]{Institute of Computer Science, University of Wroc\l{}aw, Wroc\l{}aw, Poland}
\date{}

\newcommand{\colorset}{\chi}
\newcommand{\ecc}{\textsc{ECC}\xspace}
\newcommand{\localecc}{\textsc{Local ECC}\xspace}
\newcommand{\local}{\textsc{Local}\xspace}
\newcommand{\robustecc}{\textsc{Robust ECC}\xspace}
\newcommand{\robust}{\textsc{Robust}\xspace}
\newcommand{\globalecc}{\textsc{Global ECC}\xspace}
\newcommand{\globall}{\textsc{Global}\xspace} 

\newcommand{\blocal}{b_\mathsf{local}}
\newcommand{\bglobal}{b_\mathsf{global}}
\newcommand{\brobust}{b_\mathsf{robust}}

\newcommand{\avgcolordeg}{\bar d_\colorset}
\newcommand{\maxcolordeg}{\Delta_\colorset}
\newcommand{\intersectratio}{\rho}

\newcommand{\slack}{\mathsf{slack}}
\newcommand{\rate}{\mathsf{rate}}
\newcommand{\budget}{\mathsf{budget}}
\newcommand{\localslack}{\kappa}
\newcommand{\inctime}{t_\mathsf{tighten}}

\newcommand{\ALG}{\mathsf{ALG}}
\newcommand{\OPT}{\mathsf{OPT}}

\newcommand{\EkVC}{\textsc{E$k$-Vertex-Cover}\xspace}

\newcommand{\NP}{\textnormal{NP}\xspace}

\newcommand{\brain}{$\mathsf{Brain}$\xspace}
\newcommand{\cooking}{$\mathsf{Cooking}$\xspace}
\newcommand{\magten}{$\mathsf{MAG}\textnormal{-}\mathsf{10}$\xspace}
\newcommand{\dawn}{$\mathsf{DAWN}$\xspace}
\newcommand{\walmart}{$\mathsf{Walmart}$\xspace}
\newcommand{\trivago}{$\mathsf{Trivago}$\xspace}

\begin{document}
\maketitle

\begin{abstract}
Clustering is a fundamental task in both machine learning and data mining. Among various methods, edge-colored clustering (\ecc) has emerged as a useful approach for handling categorical data. Given a hypergraph with (hyper)edges labeled by colors, \ecc aims to assign vertex colors to minimize the number of edges where the vertex color differs from the edge's color. However, traditional \ecc has inherent limitations, as it enforces a nonoverlapping and exhaustive clustering. To tackle these limitations, three versions of \ecc have been studied: \localecc and \globalecc, which allow overlapping clusters, and \robustecc, which accounts for vertex outliers. For these problems, both linear programming (LP) rounding algorithms and greedy combinatorial algorithms have been proposed. While these LP-rounding algorithms provide high-quality solutions, they demand substantial computation time; the greedy algorithms, on the other hand, run very fast but often compromise solution quality. In this paper, we present an algorithmic framework that combines the strengths of LP with the computational efficiency of combinatorial algorithms. Both experimental and theoretical analyses show that our algorithms efficiently produce high-quality solutions for all three problems: \local, \globall, and \robust \ecc. We complement our algorithmic contributions with complexity-theoretic inapproximability results and integrality gap bounds, which suggest that significant theoretical improvements are unlikely. Our results also answer two open questions previously raised in the literature.
\end{abstract}

\section{Introduction}\label{sec:intro}

Clustering is a fundamental task in both machine learning and data mining~\cite{ester1996density,kassambara2017practical, ezugwu2022comprehensive}. \emph{Edge-colored clustering} (\ecc), in particular, is a useful model when interactions between the items to be clustered are represented as categorical data~\cite{angel2016clustering, amburg2020clustering}. To provide intuition, let us consider the following simple, illustrative example from prior work~\cite{whats-cooking,amburg2020clustering,klodt2021color,xiu2022chromatic, veldt2023optimal,crane2024overlapping,crane2025edge}: given a set of food ingredients, recipes that use them, and a (noisy) labeling of these recipes indicating their cuisine (e.g., Italian or Indian), can we group the food ingredients by their cuisine? To address this question, we can begin by considering a hypergraph whose vertices correspond to ingredients, (hyper)edges represent recipes, and edge colors correspond to cuisines. We can then find a labeling of the ingredients such that, in most recipes, all ingredient labels match the recipe's label. This is precisely what \ecc does: given an edge-colored hypergraph, the goal is to assign colors to its vertices so that the number of edges where vertex colors differ from the edge color is minimized. Intuitively, this problem offers an approach for clustering vertices when edge labels are noisy.

However, \ecc has an inherent limitation in that it insists on assigning exactly one color to every vertex, enforcing a nonoverlapping and exhaustive clustering. In the above illustrative example, food ingredients are often shared across geographically neighboring cuisines, indicating that overlapping clustering may be preferable. Moreover, some ingredients, such as salt, commonly appear in nearly all cuisines and may be considered outliers that should ideally be excluded from the clustering process. To address these limitations, three generalizations of \ecc, namely, \localecc, \globalecc, and \robustecc, have been proposed~\cite{crane2024overlapping}. Among them, \localecc and \globalecc allows overlapping clustering: in \localecc, a \emph{local budget} $\blocal$ that specifies the maximum number of colors each vertex can receive is given as an input parameter, thereby allowing clusters to overlap. In \globalecc, vertices may be assigned multiple colors, but with the total number of extra assignments constrained by a \emph{global budget} $\bglobal$ given as input. On the other hand, \robustecc enhances robustness against vertex outliers by allowing up to $\brobust$ vertices to be deleted from the hypergraph. This budget $\brobust$ is also specified as part of the input. (Alternatively, this can be viewed as designating those vertices as ``wildcards'' that can be treated as any color.)

\vspace{1ex}
While \localecc, \globalecc, and \robustecc are useful extensions of \ecc that effectively address its limitations, these problems are unfortunately \NP-hard, making exact solutions computationally intractable. This directly follows from the \NP-hardness of \ecc~\cite{angel2016clustering}, a common special case of all three problems. This computational intractability naturally motivates the study of approximation algorithms for these problems. Recall that an algorithm is called a \emph{$\rho$-approximation algorithm} if it runs in polynomial time and guarantees a solution within a factor of $\rho$ relative to the optimum.

\vspace{1ex}
In this paper, we present a new algorithmic framework for overlapping and robust clustering of edge-colored hypergraphs that is linear programming-based (LP-based) yet also combinatorial. Previously, combinatorial algorithms and (non-combinatorial) LP-based algorithms have been proposed for these problems. For \localecc, Crane et al.~\citep{crane2024overlapping} gave a greedy combinatorial $r$-approximation algorithm, where $r$ is the rank of the hypergraph. Their computational evaluation demonstrated that this algorithm runs remarkably faster than their own LP-rounding algorithm, at the expense of a trade-off in solution quality. The theoretical analysis~\citep{crane2024overlapping} of the LP-rounding algorithm successfully obtains an approximation ratio that does not depend on $r$: they showed that their algorithm is a $(\blocal+1)$-approximation algorithm. They state it as an open question whether there exists an $O(1)$-approximation algorithm for \localecc. For \robustecc as well, Crane et al.\ gave a greedy $r$-approximation algorithm; however, their LP-rounding algorithm in this case does not guarantee solution feasibility. According to their computational evaluation, solutions produced by the LP-rounding algorithm were of very high quality but violated the budget constraint, which is reflected in the theoretical result: their algorithm is a \emph{bicriteria} $(2+\epsilon,2+\frac{4}{\epsilon})$-approximation algorithm for any positive $\epsilon$, i.e., an algorithm that produces an $(2+\epsilon)$-approximation solution but violates the budget constraint by a multiplicative factor of at most $2+\frac{4}{\epsilon}$. Finally for \globalecc, Crane et al.\ gave similar results: a greedy $r$-approximation algorithm and a bicriteria $(\bglobal+3+\epsilon,1+\frac{\bglobal+2}{\epsilon})$-approximation algorithm for any positive $\epsilon$, where the latter, empirically, was slow but produced solutions of high quality. Since their bicriteria approximation ratio is not $(O(1),O(1))$ for \globalecc, Crane et al.\ left it another open question whether bicriteria $(O(1),O(1))$-approximation is possible for \globalecc.

\vspace{1ex}
The primal-dual method is an algorithmic approach that constructs combinatorial algorithms based on LP, allowing one to combine the strengths of both worlds~\cite{goemans1995general,goemans1997primal}. Our algorithmic framework is designed using the primal-dual method. We analyze its performance both experimentally and theoretically. For \localecc, our framework yields a combinatorial $(\blocal+1)$-approximation algorithm, which is the same approximation ratio as Crane et al.'s LP-rounding algorithm; however, our algorithm is combinatorial and runs in linear time. The experiments confirmed that, compared to the previous combinatorial algorithm, our algorithm brings improvement in both computation time and solution quality. We complement this algorithmic result by showing inapproximability results that match our approximation ratio; this answers one of Crane et al.'s open questions. For \robustecc and \globalecc, our framework gives a true (non-bicriteria) approximation algorithm, avoiding the need for bicriteria approximation.\footnote{If, in some contexts, a bicriteria approximation algorithm is acceptable for use, we could instead use a true approximation algorithm with a relaxed budget. Thus, once a true approximation algorithm becomes available, the need for bicriteria approximation algorithms is reduced. However, our algorithmic framework can also be analyzed in the bicriteria setting for both \globalecc and \robustecc. See Appendices~\ref{app:global-bicriteria} and~\ref{app:robust-bicriteria}.} Our true approximation algorithm for \robustecc, with the ratio of $2(\brobust+1)$, was enabled by our new LP relaxation: the integrality gap of the relaxation used by previous results is $+\infty$~\cite{crane2024overlapping}, whereas our LP has an integrality gap of $O(\brobust)$. In fact, we show that our gap is $\Theta(\brobust)$, suggesting that our ratio may be asymptotically the best one can achieve based on this relaxation. For \globalecc, our true approximation algorithm has the ratio of $2(\bglobal+1)$, and our bicriteria approximation algorithm has the ratio of $(2+\epsilon,1+\frac{2}{\epsilon})$. This affirmatively answers another open question of Crane et al.: bicriteria $(O(1),O(1))$-approximation for \globalecc is indeed possible. We also show that our relaxation has the integrality gap of $\Theta(\bglobal)$.

Below, we summarize which contributions of our work are presented in which sections of the paper.

\begin{enumerate}[-]
\item In Section~\ref{sec:local}, we present our algorithm for \localecc; its performance is analyzed both experimentally (Section~\ref{sec:exp-local}) and theoretically (Section~\ref{sec:local} and Appendix~\ref{app:local-thm}). We also present the inapproximability result (Theorems~\ref{thm:local-inapprox-ugc} and \ref{thm:local-inapprox-nphard}) that answers Crane et al.'s open question~\cite{crane2024overlapping}, whose technical proof is deferred to Appendix~\ref{app:local-inapprox}.
\item In Section~\ref{sec:robustglobal}, we present our true approximation algorithm for \robustecc based on a new stronger LP formulation. Our algorithm's performance is analyzed both experimentally (Section~\ref{sec:exp-robustglobal}) and theoretically (Section~\ref{sec:robustglobal} and Appendix~\ref{app:robust-thm}), including an integrality gap lower bound (Section~\ref{sec:robustglobal}; note that an upper bound is implied by the proof of Theorem~\ref{thm:robust-approx}).
\item In Section~\ref{sec:robustglobal} and Appendix~\ref{app:global-alg}, we present our true approximation algorithm for \globalecc, whose performance is analyzed both experimentally (Section~\ref{sec:exp-robustglobal}) and theoretically (Appendix~\ref{app:global-thm}). This algorithm extends to the bicriteria setting (Section~\ref{sec:robustglobal} and Appendix~\ref{app:global-bicriteria}), answering another open question of Crane et al.~\cite{crane2024overlapping}.
\end{enumerate}

We note that LP-rounding algorithms based on our relaxations can match the ratios of our combinatorial true approximation algorithms. However, we omit them from this paper, as they offer no improvement in performance guarantees while requiring significantly more computation time to solve LPs.

\paragraph*{Related work.}
\ecc has been used for a variety of tasks including categorical community detection, temporal community detection~\cite{amburg2020clustering}, and diverse and experienced group discovery~\cite{amburg2022diverse}; recently, it has also been applied to fair and balanced clustering~\cite{crane2025edge}. For reasons of space, we review previous work related to it in Appendix~\ref{app:related}.
\section{Problem definitions}\label{sec:prob-defn}
In this section, we formally define the problems considered in this paper.
First, we describe the part of the input that is common to all three problems. We are given a hypergraph $H=(V,E)$ and a set $C$ of colors as input. Since $H$ is a hypergraph, we have $E\subseteq 2^V$. Each edge $e \in E$ is associated with a color $c_e \in C$.

Given a \emph{node coloring} $\sigma:V\rightarrow C$, we say an edge $e\in E$ is a \emph{mistake} if there exists a node $v\in e$ whose assigned color $\sigma(v)$ differs from $c_e$, i.e., $c_e\neq \sigma(v)$. Otherwise, we say that $e$ is \emph{satisfied}. In \localecc and \globalecc, a node coloring $\sigma:V\to 2^C$ assigns (possibly) a multiple number of colors to each node. In these problems, we say $e\in E$ is a \emph{mistake} if there exists a node $v\in e$ whose assigned color does not include $c_e$, i.e., $c_e \notin \sigma(v)$.
\begin{definition}
In \localecc, in addition to $H$, $C$, and $\{c_e\}_{e\in E}$, a \emph{local budget} $\blocal \in \mathbb{Z}_{\ge 1}$ is given as input.
The goal is to find a  node coloring $\sigma:V\to 2^C$ such that $|\sigma(v)|\le \blocal$ for all $v$ to minimize the number of mistakes.
\end{definition}
\begin{definition}
	In \globalecc, in addition to $H$, $C$, and $\{c_e\}_{e\in E}$, a \emph{global budget} $\bglobal \in \mathbb{Z}_{\ge 0}$ is given as input.
	The goal is to find a node coloring $\sigma:V\to 2^C$ such that $|\sigma(v)|\ge 1$ for all $v$ and $\sum_{v\in V}|\sigma(v)|\le |V|+\bglobal$, to minimize the number of mistakes.
\end{definition}
\begin{definition}
	In \robustecc, in addition to $H$, $C$, and $\{c_e\}_{e\in E}$, a \emph{node-removal budget} $\brobust \in \mathbb{Z}_{\ge 0}$ is given as input.
	The goal is to remove at most $\brobust$ nodes from the hypergraph and find a node coloring $\sigma:(V\setminus V_R)\to C$ to minimize the number of mistakes, where $V_R$ denotes the set of removed nodes.
\end{definition}
Recall that removing a node from $H$ makes the node disappear from all the incident edges.

We conclude this section by introducing notation to be used throughout this paper. For $F \subseteq E$, let $\colorset(F):=
\{c_e\mid e\in F\}$ be the set of colors of the edges in $F$. For $v \in V$, let $\delta(v)$ be the set of edges that are incident with $v$; $d_v:=|\delta(v)|$ is the degree of $v$. Let $\delta_c(v)$ be the set of edges in $\delta(v)$ whose color is $c$, i.e., $\delta_c(v):=\{e \in \delta(v) \mid c_e = c\}$. 
\section{Proposed algorithms}\label{sec:alg}
\subsection{Local ECC}\label{sec:local}
In this section, we informally present our approximation algorithm for \localecc. Although we will discuss all the necessary technical details here, we will still present a formal analysis in Appendix~\ref{app:local} for the completeness' sake.

The algorithm will be presented for a slightly different version of the problem: instead of  $\blocal$ that uniformly applies to all nodes, we will let each node $v$ specify its own budget $b_v$. We also introduce edge weights $w_e \in \mathbb{Q}_{\ge 0}$  so that we minimize the total weight, not number, of mistakes. Note that it suffices to solve this version of the problem, since we can simply set all $b_v$ as $\blocal$ and all $w_e$ as $1$.

Following are an LP relaxation (left) and its dual (right). 
Intuitively, $x_{v,c}=1$ indicates that node $v$ is colored with $c$ and $x_{v,c}=0$ otherwise; $y_e=1$ if $e$ is a mistake and $y_e=0$ otherwise.
\begin{align*}
    \text{min } & \textstyle \sum_{e \in E}{w_e y_e} &                                  && \quad \text{max } & \textstyle \sum_{e \in E, v \in e}{\beta_{e,v}}-\sum_{v\in V}{b_v\alpha_v} \\
    \text{s.t.\ } & \textstyle \sum_{c \in C}{x_{v,c}} \le b_v, & \forall v \in V,      && \quad \text{s.t.\ } & \textstyle \sum_{e\in \delta_c(v)}{\beta_{e,v}} \le \alpha_v, && \forall v \in V, c \in C, \\
                & x_{v,c_e} + y_e \ge 1, & \forall e \in E, v \in e,                    && \quad & \textstyle  \sum_{v \in e}{\beta_{e,v}} \le w_e, && \forall e \in E, \\
                & x_{v,c} \ge 0, & \forall v \in V, c \in C,                            && \quad & \alpha_v \ge 0, && \forall v \in V, \\
                & y_e \ge 0, & \forall e \in E.                                         && \quad & \beta_{e,v} \ge 0, && \forall e \in E, v \in e.
\end{align*}
As a primal-dual algorithm, our algorithm maintains a dual solution $(\alpha,\beta)$, which changes throughout the execution of the algorithm but remains feasible at all times. The algorithm constructs the ``primal'' solution partially guided by the complementary slackness: namely, it allows an edge $e$ to be a mistake only if the corresponding dual constraint $\sum_{v\in e}\beta_{e,v} \leq w_e$ is \emph{tight}, i.e., $\sum_{v\in e}\beta_{e,v} =w_e$. This is useful since the cost of the algorithm's output can then be written as $\sum_{e\in E_m}w_e= \sum_{e\in E_m}\sum_{v\in e}\beta_{e,v} \leq \sum_{e\in E}\sum_{v\in e}\beta_{e,v}$, where $E_m$ is the set of mistakes in the output. Let $B_v:=\sum_{e \in \delta(v)}\beta_{e,v}$, and the algorithm's output cost is no greater than $\sum_{v\in V}B_v$ at termination.

In order to maintain dual feasibility, the algorithm begins with a trivial dual feasible solution $(\alpha,\beta)=(\mathbf{0},\mathbf{0})$ and only increases dual variables, never decreasing them. The first set of constraints will never be violated because whenever we increase $\sum_{e\in\delta_c(v)}\beta_{e,v}$, we will increase $\alpha_v$ by the same amount. The second set of constraints will never be violated simply because we will stop increasing all $\beta_{e,v}$ for $v\in e$ once edge $e$ becomes tight.

We are now ready to present the algorithm. We will describe the algorithm as if it is a ``continuous'' process that continuously increases a set of variables as time progresses. In this \emph{process over time} perspective (see, e.g.,~\cite{grandoni2008primal}), a primal-dual algorithm starts with an initial (usually all-zero) dual solution at time $0$, and the algorithm specifies the increase rate at which each dual variable increases. The dual variables continue to increase at the specified rates until an event of interest---typically, a dual constraint becomes tight---occurs. At that point, the algorithm pauses the progression of time to handle the event and recompute the increase rates. Once updated, time proceeds again.

Consider the following algorithm. It maintains a set $L$ of all those edges that are not tight. We call these edges \emph{loose}. One point that requires additional explanation in this pseudocode is that it increases \emph{a sum of variables} $\sum_{e\in\delta_c(v)\cap L} \beta_{e,v}$ at unit rate, rather than a single variable. This should be interpreted as increasing the variables in the summation in an arbitrary way, provided that their total increase rate is 1 and that no variable is ever decreased. The algorithm's analysis holds for any such choice of the increase rates of individual variables as long as their total is 1.

\begin{algorithm}[H]
    \caption{Proposed algorithm for \localecc}\label{alg:localecc}
\begin{algorithmic}
     \STATE ${\alpha} \gets \mathbf{0};~{\beta} \gets \mathbf{0}$
    \STATE $L \gets \{e \in E \mid w_e > 0\}$
    \FOR{$v \in V$}
        \WHILE{$|\colorset(\delta(v)\cap L)| > b_v$}
            \STATE increase $\alpha_v$ and $\sum_{e\in \delta_c(v) \cap L}{\beta_{e,v}}$ for each $c \in \colorset(\delta(v)\cap L)$ at unit rate, until there exists $e$ such that $\sum_{u \in e}{\beta_{e,u}} = w_e$
            \STATE \textbf{if} $\exists e\ \sum_{u \in e}{\beta_{e,u}} = w_e$ \textbf{then} remove all such edges from $L$ 
        \ENDWHILE
        \STATE $\sigma(v) \gets \colorset(\delta(v)\cap L)$
    \ENDFOR
\end{algorithmic}
\end{algorithm}

This algorithm can be implemented as a usual discrete algorithm using the standard technique for emulating ``continuous'' algorithms by discretizing them. Once the increase rates are determined, the discretized algorithm computes, for each edge, after how much time the edge would become tight if we continuously and indefinitely increased the dual variables, and selects the minimum among them. That is the amount of time the emulated algorithm runs before getting paused. The discretized algorithm then handles the event, recomputes the increase rates, and repeat. See Appendix~\ref{app:local-pseudocode} for the full discretized version of the algorithm.

It is easy to see that Algorithm~\ref{alg:localecc} returns a feasible solution: we assign $\chi(\delta(v)\cap L)$ to $v$ only after ensuring $|\chi(\delta(v)\cap L)|\leq b_v$. The analysis can focus on bounding the final value of $\sum_{v\in V}B_v$: recall that it was an upper bound on the algorithm's output cost. We will compare $\sum_{v\in V}B_v$ against the dual objective value, which is a lower bound on the true optimum from the LP duality.

Both $\sum_{v\in V}B_v$ and the dual objective value change throughout the algorithm's execution. At the beginning, both are zeroes because $(\alpha,\beta)=(\mathbf{0},\mathbf{0})$. How do they change over the execution? 
In each iteration of the \textbf{while} loop, the algorithm increases 
$\alpha_v$ at unit rate and $B_v$ at rate $|\chi(\delta(v)\cap L)|$, where $v$ is the vertex being considered at the moment.
(Note that $B_v = \sum_{c\in\chi(\delta(v)\cap L)}\sum_{e \in \delta_c(v) \cap L}\beta_{e,v} + \sum_{e \in \delta(v) \setminus L} \beta_{e, v}$.) That is, at any given moment of the algorithm's execution, the rate by which $\sum_{u\in V}B_u$ gets increased is $|\chi(\delta(v)\cap L)|>b_v$, and the increase rate of the dual objective is $|\chi(\delta(v)\cap L)|-b_v$.
Note that the ratio between these two rates is $\frac{|\chi(\delta(v)\cap L)|}{|\chi(\delta(v)\cap L)|-b_v}\leq b_v+1$ since $|\chi(\delta(v)\cap L)|>b_v$. Since the upper bound on the algorithm's output and the lower bound on the true optimum were initially both zeroes and the ratio between their increase rate is no greater than $b_v+1$ at all times, the overall approximation ratio is $b_\mathsf{max}+1$ where $b_\mathsf{max}:=\max_{v\in V}b_v$. Note that $b_\mathsf{max}=\blocal$ under the original definition of \localecc.
\begin{restatable}{theorem}{localApprox}\label{thm:local-approx}
    Algorithm~\ref{alg:localecc} is a $(\blocal+1)$-approximation algorithm for \localecc.
\end{restatable}

Algorithm~\ref{alg:localecc} can be implemented to run in linear time (see Lemma~\ref{thm:local-rt} in Appendix~\ref{app:local-thm}).

Our algorithmic framework harnesses the full ``power'' of the LP relaxation, in that its approximation ratio matches the integrality gap of the relaxation. We defer the proof of Theorem~\ref{thm:local-ig} to Appendix~\ref{app:local-ig}.

\begin{restatable}{theorem}{localIG}\label{thm:local-ig}
There is a sequence of instances of \localecc such that the ratio between a fractional solution and an optimal integral solution converges to $\blocal+1$.
\end{restatable}

In fact, our inapproximability results further show that our approximation ratio is essentially the best possible. We note that these results answer one of the open questions raised by Crane et al.~\citep{crane2024overlapping}, namely,  whether an $O(1)$-approximation algorithm is possible for \localecc. 
\begin{restatable}{theorem}{localInapproxUGC}\label{thm:local-inapprox-ugc}
    For any constant $\epsilon > 0$, it is $\mathrm{UGC}$-hard to approximate \localecc within a factor of $\blocal+1-\epsilon$.
\end{restatable}
If one prefers a milder complexity-theoretic assumption, we show the following theorem as well.
\begin{restatable}{theorem}{localInapproxNPhard}\label{thm:local-inapprox-nphard}
    For any $\blocal \geq 2$ and any constant $\epsilon > 0$, there does not exist a $(\blocal - \epsilon)$-approximation algorithm for \localecc unless $\mathrm{P}=\mathrm{NP}$.
\end{restatable}
The proofs of Theorems~\ref{thm:local-inapprox-ugc} and~\ref{thm:local-inapprox-nphard}  are deferred to Appendix~\ref{app:local-inapprox}.

\paragraph*{Final remarks.} 
Since our algorithm considers the nodes one by one and operates locally, 
Algorithm~\ref{alg:localecc} immediately works as an online algorithm, in which vertices are revealed to the algorithm in an online manner.\footnote{There are several ways to describe the online setting, but a simple (albeit slightly weak) formulation is as follows: initially, the algorithm is given only the number of hyperedges. Then, at each timestep, when a new vertex $v$ arrives, the algorithm is informed of which hyperedges are incident to $v$. At that point of time, the algorithm is required to irrevocably color $v$.}
In Appendix~\ref{app:local-bicriteria}, we also show that the algorithm can be analyzed in the bicriteria setting, yielding a $(1+\epsilon, 1+\frac{1}{\blocal}\lceil \frac{\blocal}{\epsilon} \rceil -\frac{1}{\blocal})$-approximation for $\epsilon \in (0,\blocal]$.

\subsection{Robust ECC and Global ECC}\label{sec:robustglobal}
In this section, we summarize our algorithmic results for \robustecc and \globalecc. The proposed approximation algorithms for these two problems are quite similar; as such, in the interest of space, we will sketch our algorithm only for \robustecc in this section. The only real difference between the two algorithms is in the constraints of the dual LPs.

Following is the dual LP used by the algorithm for \robustecc. (As {in Section~\ref{sec:local}}, edges have weights $w_e$, but we can simply set all $w_e$ as $1$.)

\begin{wrapfigure}{l}{0.46\textwidth}
    \vspace{-1em}
    \begin{minipage}{0.45\textwidth} 
        \begin{align*}
            \text{max } &\textstyle \sum_{e \in E, v \in e}{\beta_{e,v}}-\sum_{v\in V}&&\hspace{-10pt}{\alpha_v}-\lambda \brobust \\
            \text{s.t.\ } &\textstyle \sum_{e\in \delta_c(v)}{\beta_{e,v}} \le \alpha_v, &&\hspace{-2pt} \forall v \in V, c \in C, \\
                        &\textstyle \sum_{v \in e}{\beta_{e,v}} \le w_e, &&\hspace{-2pt} \forall e \in E, \\
                        &\textstyle \sum_{e \in \delta(v)}{\beta_{e,v}}-\alpha_v \le \lambda, &&\hspace{-2pt} \forall v \in V, \\
                        & \alpha_v \ge 0, &&\hspace{-2pt} \forall v \in V, \\
                        & \beta_{e,v} \ge 0, &&\hspace{-2pt} \forall e \in E, v \in e, \\
                        & \lambda \ge 0.
        \end{align*}
    \end{minipage}
\end{wrapfigure}
Let us now sketch the algorithm {for \robustecc} we propose. The algorithm maintains a dual feasible solution $(\alpha, \beta, \lambda)$, initially set as $(\mathbf{0}, \mathbf{0}, 0)$. The set $L$ will be kept as the set of loose edges;
$R \subseteq V$ is the set of nodes with at least two incident loose edges of distinct colors.
Intuitively, $R$ is the set of nodes we will remove from the hypergraph. The algorithm therefore continues its execution until $|R|\leq \brobust$ holds. When increasing the dual variables, the algorithm increases variables associated with all vertices in $R$ at the same time, unlike Algorithm~\ref{alg:localecc} which handles one node at a time.
The following two properties will hold: 

\begin{enumerate}[(i)]
\item \label{enum:robalg-a} The algorithm increases $\lambda$ and $\sum_{e \in \delta(v)\cap L}{\beta_{e,v}}-\alpha_v$ for each $v \in R$ at the same rate.
\item \label{enum:robalg-b} For each $v\in R$, the algorithm increases $\alpha_v$ and $\sum_{e \in \delta_c(v)\cap L}{\beta_{e,v}}$ for each $c \in \colorset(\delta(v) \cap L)$ at the same rate.  In general, the increase rate of $\alpha_{v_1}$ may be different from that of $\alpha_{v_2}$ for $v_1\neq v_2$.
\end{enumerate}

These properties can be ensured as follows:
the increase rate of $\lambda$ is set as $1$. For each $v\in R$,
we increase $\alpha_v$ and $\sum_{e \in \delta_c(v) \cap L} \beta_{e,v}$ for each $c \in \colorset(\delta(v) \cap L)$ at rate $\frac{1}{|\colorset(\delta(v)\cap L)|-1}$.

Once $|R|$ becomes less than or equal to $\brobust$, the algorithm removes $R$ from the hypergraph and assigns every node $v \in V \setminus R$ the (only) color in $\colorset(\delta(v) \cap L)$. If $\colorset(\delta(v) \cap L) = \emptyset$, an arbitrary color can be assigned without affecting the theoretical guarantee on solution quality; in practical implementation, we could employ heuristics for marginal improvement. In the interest of space, the full pseudocodes are deferred to Appendices~\ref{app:robust-alg} and~\ref{app:robust-pseudocode}.

We prove the following theorems in Appendices~\ref{app:robust-thm} and~\ref{app:global-thm}.
\begin{restatable}{theorem}{robustApprox}\label{thm:robust-approx}
    There exists a $2(\brobust+1)$-approximation algorithm for \robustecc.
\end{restatable}
\begin{restatable}{theorem}{globalApprox}\label{thm:global-approx}
    There exists a $2(\bglobal+1)$-approximation algorithm for \globalecc.
\end{restatable}

The LP relaxation of Crane et al.~\citep{crane2024overlapping} for \robustecc has infinite integrality gap, whereas the integrality gap of our LP is $O(\brobust)$, following from the proof of Theorem~\ref{thm:robust-approx}. This makes it possible to obtain a true (non-bicriteria) approximation algorithm based on our LP. In fact, the following theorems show that our LP for \robustecc (and \globalecc) has an integrality gap of $\Theta(\brobust)$ (and $\Theta(\bglobal)$), respectively. Their proofs are deferred to Appendices~\ref{app:robust-ig} and~\ref{app:global-ig}. 
\begin{restatable}{theorem}{robustIG}\label{thm:robust-ig}
    The integrality gap of our LP for \robustecc is at least $\brobust+1$.
\end{restatable}
\begin{restatable}{theorem}{globalIG}\label{thm:global-ig}
    The integrality gap of the LP for \globalecc is at least $\bglobal+1$.
\end{restatable}
\paragraph*{Final remarks.} Our algorithms can be analyzed in the bicriteria setting as well, yielding a bicriteria $(2+\epsilon,1+\frac{1}{b}\lceil \frac{2b}{\epsilon} \rceil -\frac{1}{b})$-approximation algorithm for all $\epsilon \in (0,2b]$, {where $b=\brobust$ for \robustecc and $b=\bglobal$ for \globalecc}.
This improves the best bicriteria approximation ratios previously known; furthermore, it affirmatively answers one of the open questions of Crane et al.~\cite{crane2024overlapping}, namely, whether there exists a  bicriteria $(O(1),O(1))$-approximation algorithm for \globalecc. See Appendices~\ref{app:robust-bicriteria} and \ref{app:global-bicriteria}.

\section{Experiments}\label{sec:exp}
In this section, we analyze the performance of our algorithmic framework through experiments.
We describe the experimental setup in Section~\ref{sec:exp-setup}.
We evaluate and discuss the performance of our algorithm for \localecc in Section~\ref{sec:exp-local}. In Section~\ref{sec:exp-robustglobal}, we address \textsc{Robust} and \globalecc.

\subsection{Setup}\label{sec:exp-setup}
Our experiments  used the same benchmark as Crane et al.~\citep{crane2024overlapping}, which contains six datasets. See Appendix~\ref{app:data} for further description of the individual datasets. We remark that these  datasets have  been used as a benchmark to experimentally evaluate \ecc also in other prior work~\cite{amburg2020clustering, veldt2023optimal}.
Table~\ref{tab:dataset} summarizes some statistics of the datasets:
the number of nodes $|V|$,  number of edges $|E|$,  number of colors $|C|$,  rank $r:=\max_{e\in E}|e|$,  average degree $\bar d:= \sum_{v\in V}d_v/|V|$, maximum color-degree $\maxcolordeg := \max_{v\in V}|\colorset(\delta(v))|$, average color-degree $\avgcolordeg := \sum_{v\in V}|\colorset(\delta(v))|/|V|$, and the ratio $\intersectratio$ of vertices whose color degree is at least 2, i.e., $\intersectratio:=|\{v\in V \mid |\colorset(\delta(v))| \ge 2\}|/|V|$.

\begin{table}
\caption{Statistics of the benchmark datasets.}\label{tab:dataset}
\centering
\begin{tabular}{crrrrrrrr}
    \toprule
    \multicolumn{1}{c}{Datasets} & \multicolumn{1}{c}{$|V|$} & \multicolumn{1}{c}{$|E|$} & \multicolumn{1}{c}{$|C|$} & \multicolumn{1}{c}{$r$} & \multicolumn{1}{c}{$\bar d$} & \multicolumn{1}{c}{$\maxcolordeg$} & \multicolumn{1}{c}{$\avgcolordeg$} & \multicolumn{1}{c}{$\intersectratio$}\\
    \midrule
    \brain & 638 & 21,180 & 2 & 2 & 66.4 & 2 & 1.92 & 0.91\\
    \magten & 80,198 & 51,889 & 10 & 25 & 2.3 & 9 & 1.26 & 0.18\\
    \cooking & 6,714 & 39,774 & 20 & 65 & 63.8 & 20 & 4.35 & 0.61\\
    \dawn & 2,109 & 87,104 & 10 & 22 & 162.7 & 10 & 3.72 & 0.74\\
    \walmart & 88,837 & 65,898 & 44 & 25 & 5.1 & 40 & 2.65 & 0.52\\
    \trivago & 207,974 & 247,362 & 55 & 85 & 3.6 & 32 & 1.55 & 0.23\\
    \bottomrule
\end{tabular}

\end{table}

All experiments were performed on a machine with Intel Core i9-9900K CPU and 64GB of RAM.
In our experiments, we used the original code of Crane et al.~\citep{TheoryInPractice,crane2024overlapping} as the implementation of the previous algorithms.
Since their code was written in Julia, we implemented our algorithms also in Julia to ensure a fair comparison.
When running the original codes for the LP-rounding algorithms, we used Gurobi-12.0 as the LP solver. Gurobi was the solver of choice  in previous work~\citep{TheoryInPractice,crane2024overlapping,veldt2023optimal,amburg2020clustering}, and it is widely recognized for its excellent speed~\cite{mittelmann2018benchmark,mittelmann2023progress}.

Our experiments  focus on two aspects of the algorithms' performance: solution quality and running time. To compare solution quality, we will use \emph{relative error estimate}, a normalized, estimated error of the algorithm's output cost (or quality) compared to the optimum.
Since the problems are \NP-hard, it is hard to compute the exact error compared to the optimum; as such, Crane et al.~\citep{crane2024overlapping} used the optimal solution to their LP relaxation in lieu of the true optimum, giving an overestimate of the error. We followed this approach, but we used our LP relaxation instead since we can prove that our relaxation always yields a better estimate of the true optimum. To normalize the estimated error, we divide it by the estimated optimum: that is, the relative error estimate is defined as $(A-L)/L$, where $A$ denotes the algorithm's output cost and $L$ is the LP optimum.\footnote{{When $L=0$, we define the relative error estimate as $0$. Note that $L=0$ implies $A=0$ since our LP has a bounded integrality gap.}}

Crane et al.'s experiment~\citep{crane2024overlapping} 
used $\blocal \in \{1,2,3,4,5,8,16,32\}$ for \localecc,
$\brobust/|V| \in \{0, .01, .05,$ $.1, .15, .2, .25\}$ for \robustecc,
and $\bglobal/|V| \in \{0, .5, 1, 1.5, 2, 2.5, 3, 3.5, 4\}$ for \globalecc.
While these choices were carefully made so that we can avoid \emph{trivial} instances, we decided to extend their choice for \globalecc. To explain what  trivial instances are, suppose that $\blocal$ is greater than  the maximum color-degree $\maxcolordeg$ in an instance of \localecc. The problem then becomes trivial, since the local budget  allows assigning each vertex \emph{all} the colors of its incident edges. We call an instance of \localecc \emph{trivial} if $\blocal\geq\maxcolordeg$; similarly, \robustecc instances are trivial if $\brobust \ge \intersectratio|V|$, and \globalecc instances are trivial if $\bglobal \ge |V|(\avgcolordeg-1)$.
For \localecc and \robustecc, Crane et al.'s choice of budgets ensure that most instances are nontrivial: each data set has 0, 1, or at most 2 trivial instances, possibly with the exception of at most one dataset. However, for \globalecc, only 44 instances out of 78 in the original benchmark are nontrivial, so we decided to additionally test  $\bglobal/|V| \in \{.1, .2, .3, .4\}$. As a result, we tested thirteen different budgets in total for each dataset for \globalecc.
\subsection{Local ECC}\label{sec:exp-local}
We measured the solution quality and running time of the proposed algorithm in comparison with the greedy combinatorial algorithm and the LP-rounding algorithm of Crane et al.~\citep{crane2024overlapping}.

\begin{table}
	\caption{Average running times of each dataset (in seconds): \localecc. Values in parentheses are averages excluding trivial instances.}\label{tab:exp-local-rt}
	\centering
	\begin{tabular}{cccc}
		\toprule
		& Proposed & Greedy & LP-rounding \\
		\midrule
		\brain & 0.023~~~(0.120)  & 0.007~~~(0.028) & \phantom{00}0.743~~~\phantom{00}(3.739) \\
		\magten & 0.142~~~(0.134) & 0.587~~~(0.554) & \phantom{0}10.413~~~\phantom{0}(10.677) \\
		\cooking & 0.032~~~(0.035) & 0.099~~~(0.103) & \phantom{0}39.702~~~\phantom{0}(44.916) \\
		\dawn & 0.016~~~(0.019) & 0.040~~~(0.040) & \phantom{00}3.948~~~\phantom{00}(4.658) \\
		\walmart & 0.190~~~(0.190) & 1.443~~~(1.443)  & 145.427~~~(145.427) \\
		\trivago & 0.323~~~(0.313) & 3.709~~~(3.608) & 678.585~~~(677.036) \\
		\bottomrule
	\end{tabular}
\end{table}
\begin{figure}
	\centering
	\includegraphics[width=\columnwidth]{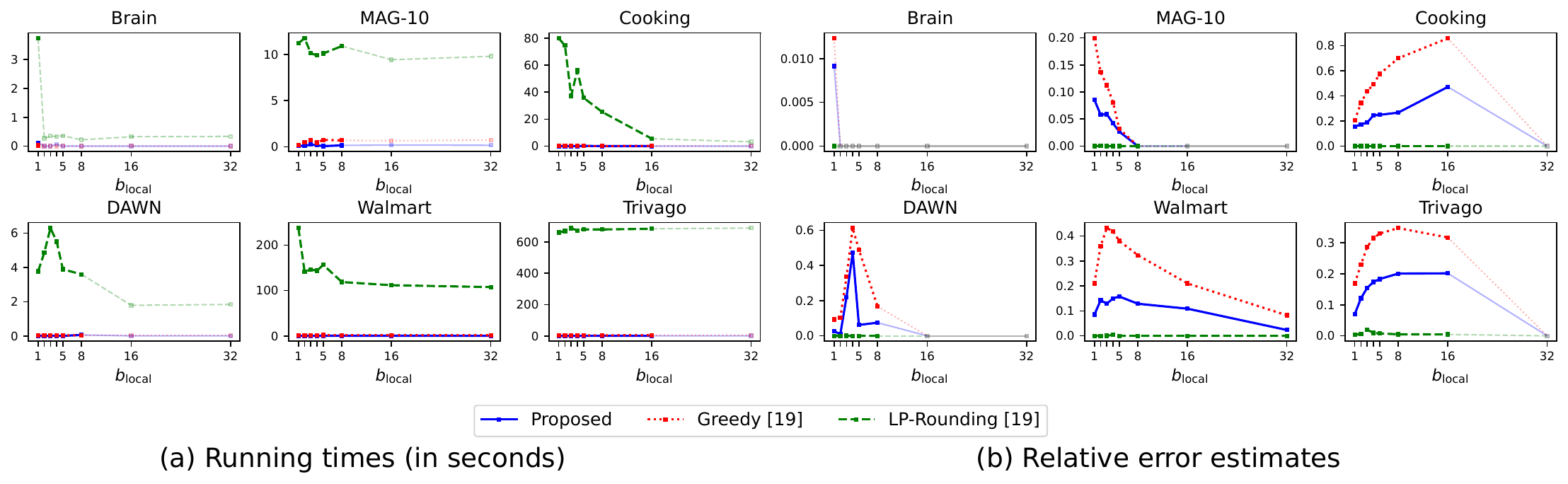}
	\caption{
		(a) Running times (in seconds) and (b) relative error estimates of the \localecc algorithms. Empty square markers denote trivial instances.
	}\label{fig:exp-local}
\end{figure}

Figure~\mbox{\ref{fig:exp-local}(a)} depicts the running times, and Table~\ref{tab:exp-local-rt} lists their average for each dataset.
Figure~\mbox{\ref{fig:exp-local}(a)} shows that our proposed algorithm was the fastest in most instances. It is not surprising that our algorithm, with the overall average running time of 0.121sec, was much faster than the LP-rounding algorithm whose overall average running time was 146.470sec, since our algorithm is combinatorial. This gap was no smaller even  when we consider only nontrivial instances: the overall average running times were 0.142sec (proposed) and 180.367sec (LP-rounding).
Remarkable was that the proposed algorithm was faster than the greedy algorithm, too. In fact, on average, it was more than twice as fast as the greedy algorithm in most datasets except for \brain.
Such gap in the running times became more outstanding in larger datasets: for \trivago, our proposed algorithm was 11 times faster than the  greedy algorithm and 2,100 times faster than the LP-rounding algorithm.

Figure~\mbox{\ref{fig:exp-local}(b)} shows the relative error estimates of the algorithms' outputs.
We note that, except for \brain and \magten, the relative error estimate of our algorithm (and of the greedy algorithm) tends to increase as $\blocal$ increases,  and then at some point starts decreasing. This appears to be the result of the fact that the problem becomes more complex as $\blocal$ initially increases, but when $\blocal$ becomes too large, the problem becomes easy again.
It can be seen from Figure~\mbox{\ref{fig:exp-local}(b)} that our proposed algorithm outperformed the greedy algorithm in all cases. 
The overall average relative error estimate of our proposed algorithm was 0.141, which is less than half of the greedy algorithm's average of 0.297.
The LP-rounding algorithm output near-optimal solutions in every case.

Overall, these experimental results demonstrate that our algorithmic framework is scalable, and produces solutions of good quality.
As was noted by Veldt~\citep{veldt2023optimal} and observed in this section, LP-rounding approach does not scale well due to its time consumption, even though it produces near-optimal solutions when it is given sufficient amount of time. 
Compared to the greedy combinatorial algorithm, our proposed algorithm output better solutions in smaller amount of time in most cases.
This suggests that the proposed algorithm can provide improvement upon the greedy algorithm.
\subsection{Robust ECC and Global ECC}\label{sec:exp-robustglobal}
Since the proposed algorithms for \robustecc and \globalecc are similar, we present the experimental results of both problems together in this section, starting with \robustecc.

We measured the performance of our proposed algorithm in addition to the greedy combinatorial algorithm and the LP-rouding algorithm of Crane et al.~\citep{crane2024overlapping}.
However, as their LP-rounding algorithm is a bicriteria approximation algorithm that possibly violates the budget $\brobust$, we cannot directly compare their solution quality with the proposed algorithm.
In fact, the LP-rounding algorithm turned out to output ``superoptimal'' solutions violating $\brobust$ in most cases of the experiment. The bicriteria approximation ratio was chosen as $(6,3)$, which is the same choice as in Crane et al.'s experiment~\citep{crane2024overlapping}.\footnote{As a side remark, when we reran the proposed algorithm with the budget tripled to enable a comparison with the LP-rounding $(6,3)$-approximation algorithm, the number of mistakes made by the proposed algorithm was, on average, as small as 57.2\% of that made by the bicriteria algorithm.} 

Comparing the average running times of each dataset reveals that the proposed algorithm ran much faster than the LP-rounding algorithm for most datasets, except for \dawn. The proposed algorithm was slower than the greedy algorithm for all datasets; however, it tended to produce solutions of much better quality than the greedy algorithm. The relative error estimate of the proposed algorithm was strictly better than that of the greedy algorithm in all nontrivial instances; the overall average relative error estimate of the proposed algorithm  was 0.042, six times better than the greedy algorithm's average of 0.272. We also note that the relative error estimate of our algorithm stayed relatively even regardless of the budget, while that of the greedy algorithm fluctuated as $\brobust$ changed in some datasets, such as \magten and \trivago. Due to space constraints, a detailed table and a figure presenting the experimental results have been deferred to Appendix~\ref{app:exp-graphs}.

For \globalecc, the bicriteria approximation ratio of the LP-rounding algorithm was chosen as $(2\bglobal+5, 2)$, which again is the same choice as in Crane et al.'s experiment. For \globalecc, the bicriteria approximation algorithm did not violate the budget for any instances of the benchmark. This may be due to the fact that their LP relaxation for \globalecc has a bounded integrality gap, unlike their LP for \robustecc.\footnote{When we reran the proposed algorithm with the budget doubled, the number of mistakes made by the proposed algorithm was, on average, as small as 68.9\% of that made by the bicriteria $(2\bglobal+5, 2)$-approximation algorithm.}

The experimental results for Global ECC exhibited similar trends to those for Robust ECC.
The relative error estimate of the proposed algorithm was strictly better than that of the greedy algorithm in all nontrivial instances. The average relative error estimate on nontrivial instances was 0.039 for the proposed algorithm, while that of the greedy algorithm was 0.912---more than 23 times higher.
We also note that the relative error estimate of the greedy algorithm rapidly increased as $\bglobal$ increased.
While the proposed algorithm was on average slower than the greedy algorithm for all datasets, it was much faster than the LP-rounding algorithm in all datasets except for \dawn.
A detailed table and a figure presenting the experimental results have been again deferred to Appendix~\ref{app:exp-graphs}  due to the space constraints.

The above results together indicate that our proposed algorithms for \robustecc and \globalecc are likely to be preferable when a high-quality solution is desired possibly at the expense of a small increase in  computation time.

\section{Conclusion and discussion}\label{sec:conclude}
In this paper, we presented a new algorithmic framework for overlapping and robust clustering of edge-colored hypergraphs.
Experimental results demonstrated that our framework improves upon the previous combinatorial algorithm for \localecc in both computation time and solution quality; compared to LP-rounding, it achieves significantly faster computation, with a slight trade-off in solution quality.
For \robustecc and \globalecc, our framework delivers improved solution quality with a slight increase in computation time compared to the previous combinatorial algorithms, while strictly satisfying the budget constraint.
On the theoretical side, our analyses show that we achieve true $(\blocal+1)$-, $2(\brobust+1)$-, $2(\bglobal+1)$-approximation for \textsc{Local}, \textsc{Robust}, and \globalecc, respectively. We also provide
inapproximability results for \localecc and integrality gap results for all three problems, suggesting that significant theoretical improvements are unlikely. These  results lead to answers to two open questions posed in the literature~\citep{crane2024overlapping}.

There remain a few promising directions for future research. Although our combinatorial algorithm runs significantly faster than LP-rounding algorithms, its running time is still superlinear for \robustecc and \globalecc. Can we optimize the dual update steps of our algorithms to obtain a linear-time algorithm for these two problems? Also, while our work focused on giving a better algorithm for \ecc, it would be also interesting to explore additional applications of \ecc, e.g., to the clustering tasks  solved via correlation clustering (see Appendix~\ref{app:related}). Given that $k$-\textsc{Partial Vertex Cover} admits a $2$-approximation algorithm~\cite{gandhi2004approximation}, another interesting question is if we can obtain an $O(1)$-approximation algorithm for \robustecc as well.

\bibliography{lit}
\bibliographystyle{plain}

\appendix
\section{Related work}\label{app:related}

Angel et al.~\citep{angel2016clustering} initiated the study of clustering edge-colored graphs (not hypergraphs). After showing its NP-hardness, they gave the first approximation algorithm for the (maximization) problem, with the approximation ratio of $e^{-2}$. Subsequent studies~\cite{ageev2014improved, alhamdan2019approximability, ageev20200} improved this ratio, and recently, Crane et al.~\citep{crane2025edge} achieved $\frac{154}{405}$-approximation.

Given the emerging importance of clustering data with higher-order interactions~\cite{benson2016higher, li2017inhomogeneous}, Amburg et al.~\citep{amburg2020clustering} addressed clustering on edge-colored \emph{hypergraphs} for the first time, and gave  $2$-approximation algorithms.  Veldt~\citep{veldt2023optimal} presented a combinatorial $2$-approximation algorithm along with a UGC-hardness ruling out any constant smaller than $2$.

As was highlighted by previous studies~\citep{anava2015improved,amburg2020clustering,veldt2023optimal}, \ecc is closely related to correlation clustering problems~\cite{bansal2004correlation}, which has been extensively studied in machine learning and data mining~\cite{yarkony2012fast, beier2014cut, pandove2018correlation, wahid2022literature}. 
They share the common feature of taking (hyper)edges representing similarity between vertices as input, and thus both have been applied to similar sets of taks such as community detection~\citep{veldt2018correlation, amburg2020clustering}. However, correlation clustering differs from \ecc in that it treats the absence of an edge as an indication of dissimilarity, whereas \ecc interprets it merely as a lack of information.
Chromatic correlation clustering, which introduces categorical edges to correlation clustering, is another closely related problem to \ecc~\cite{bonchi2015chromatic, anava2015improved, klodt2021color, xiu2022chromatic}. Interestingly, unlike correlation clustering which was studied on hypergraphs and received significant interest~\cite{kim2011higher, fukunaga2019lp, gleich2018correlation}, it appears that the chromatic hypergraph correlation clustering has never been studied to the best of our knowledge. We note that this may be an interesting future direction of research. Other variants of correlation clustering, including 
overlapping variants~\cite{bonchi2013overlapping, andrade2014evolutionary, li2017motif, chagas2019hybrid}, and robust variants~\cite{krishnaswamy2019robust, ji2021approximation} have been studied. We refer interested readers to the book by Bonchi, Garc{\'\i}a-Soriano, and Gullo~\citep{bonchi2022correlation} and references therein.
\section{Technical details and proofs for Local ECC deferred from Section~\ref{sec:local}}\label{app:local}
\subsection{Formal proof of Theorem~\ref{thm:local-approx}}\label{app:local-thm}
\begin{lemma} \label{lem:local-key}
Algorithm~\ref{alg:localecc} satisfies the following:
\begin{enumerate}[(a)]
\item \label{enum:local-a} At any moment, $(\alpha, \beta)$ is feasible to the dual LP.
\item \label{enum:local-b} At any moment, for all $v \in V$, $\alpha_v \leq \frac{1}{b_v + 1} \sum_{e \in \delta(v)} \beta_{e,v}$.
\item \label{enum:local-c} At termination, every mistake $e$ under $\sigma$ is tight, i.e., not loose.
\end{enumerate}
\end{lemma}

\begin{proof}
	Properties~\eqref{enum:local-a} and~\eqref{enum:local-b} were shown in Section~\ref{sec:local}; let us show Property~\eqref{enum:local-c}.
    Let $e$ be an arbitrary loose edge, and suppose towards contradiction that $c_e\notin\sigma(v)$ for some $v\in e$. Observe that, once an edge becomes tight, the algorithm never makes it loose again. Therefore, $e$ was loose at the end of the iteration for $v$ of the \textbf{for} loop. Then $\delta(v)\cap L$ contained $e$ and therefore $c_e\in\colorset(\delta(v) \cap L)$, leading to contradiction.
\end{proof}
Let $\ALG$ be the total weight of mistakes in the output of Algorithm~\ref{alg:localecc} and $\OPT$ be the weight of an optimal solution.
\begin{lemma} \label{thm:local-rho}
    We have $\ALG \leq (b_\mathsf{max}+1) \cdot \OPT$.
\end{lemma}
\begin{proof}
    By Properties~\eqref{enum:local-a} and~\eqref{enum:local-b} of Lemma~\ref{lem:local-key}, we have
    \[
    \OPT \geq \sum_{e\in E}{\sum_{v \in e}{\beta_{e,v}}}-\sum_{v\in V}{ b_v \alpha_v}
      = \sum_{v\in V}{\Big( \sum_{e \in \delta(v)}{\beta_{e,v}}- b_v \alpha_v \Big)} 
       \ge \sum_{v\in V}{\Big(\frac{1}{b_v + 1}\sum_{e \in \delta(v)}{\beta_{e,v}}\Big)},
  \]
    where the first inequality is due to the (weak) LP duality. On the other hand, we have
    \[
    \ALG \leq \sum_{e\in E\setminus L}{w_e} = \sum_{e\in E\setminus L}{\sum_{v \in e}{\beta_{e,v}}} \\
    \le \sum_{e\in E}{\sum_{v \in e}{\beta_{e,v}}} = \sum_{v \in V} \sum_{e \in \delta(v)} \beta_{e, v},
    \]
    where the first inequality is due to Property~\eqref{enum:local-c}. The two inequalities together completes the proof.
\end{proof}
    
Now we need to show that the algorithm runs in polynomial time. In fact, the algorithm can be implemented to run in linear time. Recall that the size of $H$ is $\sum_{v \in V} d_v$.
\begin{lemma}\label{thm:local-rt}
    Algorithm~\ref{alg:localecc} can be implemented to run in $O(\sum_{v \in V} d_v)$ time.
\end{lemma}
\begin{proof}
    For each edge, let us maintain the ``level'' $\ell_e := \sum_{u \in e} \beta_{e, u}$. Consider an iteration for node $v \in V$. By enumerating $\delta(v)$, we can compute, for every color $c \in \colorset(\delta(v))$, the ``slack'' $\slack(c) := \sum_{e \in \delta_c(v)} (w_e - \ell_e)$. Let $c^\star$ be the color $c \in \colorset(\delta(v))$ that has the ($b_v+1$)-st largest slack. 
    Let $s^\star:=\slack(c^\star)$.
    Note that we can identify $c^\star$ in $O(d_v)$ time using the algorithm of Blum et al.~\citep{blum1972linear}. Once we found $c^\star$, for every color $c \in \colorset(\delta(v))$, we increase $\sum_{e \in \delta_c(v)} \beta_{e,v}$ by $\min\{ \slack(c), s^\star \}$  while maintaining the dual feasibility, i.e., for each edge $e \in E$, $\sum_{v \in e} \beta_{e, v} \leq w_e$ must be satisfied at the end. We then update $\{\ell_e\}_{e \in \delta_c(v)}$ accordingly. Note that a single iteration can be implemented to run in $O(d_v)$, completing the proof.
\end{proof}

\localApprox*
\begin{proof}
    Immediate from Lemmas~\ref{thm:local-rho} and~\ref{thm:local-rt}.
\end{proof}

\subsection{Proof of Theorem~\ref{thm:local-ig}}\label{app:local-ig}
\localIG*
\begin{proof}
    Consider a hypergraph $H=(V,E)$ where $|E|$ is sufficiently large and $|V|=\binom{|E|}{\blocal+1}$. All edge weights are 1. In the hypergraph, each node is uniquely labeled by a subset $S$ of $E$ such that $|S|=\blocal+1$. Let $v_S$ denote the node whose label is $S$. For each $v_S \in V$, the set of edges that are incident to $v_S$ is $S$. The colors of edges are distinct, i.e., $c_{e}\neq c_{e'}$ for all $e\neq e' \in E$. 
    
    We claim that, in any integral solution, the number of \emph{satisfied} edges (i.e., edges that are not mistakes) does not exceed $\blocal$. Suppose toward contradiction that there is a color assignment where at least $\blocal+1$ edges are satisfied. Let $S$ be any subset of the satisfied edges of size exactly $\blocal+1$. 
    Since the colors are all distinct, node $v_S$ must be colored with (at least) $\blocal+1$ colors, contradicting  the budget constraint. Therefore, the total number of mistakes of any integral solution is at least $|E|-\blocal$.

    Now consider the following fractional solution. Let $x_{v_S,c_e}$ has value $\frac{\blocal}{\blocal+1}$ if $e \in S$, otherwise $0$. Let $y_e=\frac{1}{\blocal+1}$ for all $e\in E$. Observe that the constructed solution is feasible to the LP and its cost is $\frac{|E|}{\blocal+1}$. The integrality gap is at least
    \[
        \frac{|E|-\blocal}{\frac{|E|}{\blocal+1}}=\frac{(|E|-\blocal)(\blocal+1)}{|E|}=\blocal+1-\frac{\blocal(\blocal+1)}{|E|},
    \]
    which converges to $\blocal+1$ as $|E|$ tends to infinity.
\end{proof}

\subsection{Proof of Theorems~\ref{thm:local-inapprox-ugc} and \ref{thm:local-inapprox-nphard}}\label{app:local-inapprox}

\localInapproxUGC*
\localInapproxNPhard*

We say a hypergraph $H=(V,E)$ is $k$-uniform if, for all $e\in E$, $|e|=k$. Given a $k$-uniform hypergraph $H = (V, E)$, \EkVC asks to find a minimum-size subset $S \subseteq V$ of vertices, called a \emph{vertex cover}, such that every hyperedge $e \in E$ intersects $S$, i.e., $e \cap S \neq \emptyset$ for each $e \in E$. Bansal and Khot~\citep{bansal2010inapproximability} showed the following theorem.
\begin{theorem}[Bansal and Khot~\citep{bansal2010inapproximability}] \label{thm:ekvc-ugc}
For any $k \geq 2$ and any constant $\epsilon > 0$, there does not exists a $(k-\epsilon)$-approximation algorithm for \EkVC assuming the Unique Game Conjecture.
\end{theorem}

Dinur, Guruswami, Khot, and Regev~\citep{dinur2005new} showed the following theorem.
\begin{theorem}[Dinur et al.~\citep{dinur2005new}] \label{thm:ekvc-nphard}
For any $k \geq 3$ and any constant $\epsilon > 0$, there does not exists a $(k-1-\epsilon)$-approximation algorithm for \EkVC unless $\mathrm{P}=\mathrm{NP}$.
\end{theorem}

Due to Theorems~\ref{thm:ekvc-ugc} and \ref{thm:ekvc-nphard}, it suffices to present an approximation-preserving reduction from \EkVC to \localecc with $\blocal := k-1$.
\begin{proof}[Proof of Theorems~\ref{thm:local-inapprox-ugc} and \ref{thm:local-inapprox-nphard}]
Given a $k$-uniform hypergraph $H = (W, F)$ as an input to \EkVC, let $H' := (V, E)$ be a hypergraph defined as follows:
\begin{itemize}
\item $V := \{ v_f \mid f \in F \}$ and
\item $E := \{ e_w \mid w \in W \}$ where $e_w := \{ v_f \mid f \ni w \}$.
\end{itemize}
Let $C := \{c_w \mid w \in W\}$ be a set of $|W|$ number of distinct colors.
Let us then consider the input to \localecc where $H'$ is given as the hypergraph, the color of $e_w$ is $c_w$ for every $e_w \in E$, and the budget $\blocal$ is set to $k-1$.

For any vertex cover $S \subseteq W$ in $H$, let $\sigma_S$ be the node coloring defined as follows: for every $v_f \in V$, $\sigma_S (v_f) := \{ c_w \mid w \in f \setminus S \}$.
Note that, for every $w \in W \setminus S$, $e_w$ is satisfied by $\sigma_S$.
Moreover, since $|f| = k$ and $f \cap S \neq \emptyset$, we can see that $|\sigma_S(v_f)| \leq k - 1 = \blocal$.
This shows $\sigma_S$ is indeed a feasible node coloring whose number of mistakes is at most $|S|$.
We can therefore deduce that the minimum size of a vertex cover in the original input is at least the minimum number of mistakes in the reduced input.

For the other direction, let us now consider a feasible node coloring $\sigma$.
Observe that, for any $v_f \in V$, at least one edge in $\delta(v_f)$ must be a mistake since $|\delta(v_f)| = |f| = k = \blocal+1$ and the colors of $E$ are distinct.
This implies that, for every $f \in F$, there exists a vertex $w \in W$ such that $e_w \in E$ is a mistake due to $\sigma$ in the reduced input.
This shows that, given a feasible node coloring $\sigma$ to the reduced input, we can construct in polynomial time a feasible vertex cover in the original input whose size is the same as the number of mistakes due to $\sigma$.
Together with the above argument that the minimum number of mistakes in the reduced input is at most the minimum size of a vertex cover in the original input, this implies an approximation-preserving reduction from \EkVC to \localecc.
\end{proof}

\subsection{Bicriteria algorithm for Local ECC}\label{app:local-bicriteria}
\begin{theorem}
    For any $\epsilon \in (0,\blocal]$, there exists a $(1+\epsilon,1+\frac{1}{\blocal}\lceil \frac{\blocal}{\epsilon} \rceil -\frac{1}{\blocal})$-approximation algorithm for \localecc.
\end{theorem}
\begin{proof}
    Let $\tau:=\lceil \frac{\blocal}{\epsilon} \rceil -1$.
    Consider the algorithm where the condition of \textbf{while} loop of Algorithm~\ref{alg:localecc} is replaced by $|\colorset(\delta(v)\cap L)| > \blocal + \tau$.
    Let $\sigma$ be the assignment output by the modified algorithm. Observe first that the number of colors assigned to each $v$ is at most $\blocal + \tau=\blocal\cdot(1+\frac{1}{\blocal}\lceil \frac{\blocal}{\epsilon} \rceil -\frac{1}{\blocal})$.
    Observe that Properties~\eqref{enum:local-a} and~\eqref{enum:local-c} of Lemma~\ref{lem:local-key} still hold.
    Moreover, instead of Property~\eqref{enum:local-b}, it is easy to  show a stronger property that, at any moment, for all $v \in V$,
    \begin{equation} \label{eq:app-local-bic-b}
        \alpha_v \le \frac{1}{\blocal+\tau+1}\sum_{e\in \delta(v)}{\beta_{e,v}}.
    \end{equation}
    We therefore have
    \begin{align*}
        \ALG & \le \sum_{v\in V}{\sum_{e\in \delta(v)}{\beta_{e,v}}}
        \le \sum_{v\in V}{\frac{\blocal+\tau+1}{\tau+1}\left(\sum_{e\in \delta(v)}{\beta_{e,v}}-\blocal\alpha_v\right)} \\
        & = \sum_{v\in V}{\left(1+\frac{\blocal}{\tau+1}\right)\left(\sum_{e\in \delta(v)}{\beta_{e,v}}-\blocal\alpha_v\right)} \\
        & \le \left(1+\epsilon \right) \sum_{v\in V} {\left(\sum_{e\in \delta(v)}{\beta_{e,v}}-\blocal\alpha_v\right)}
        \leq (1+\epsilon) \OPT,
    \end{align*}
    where the first inequality follows from Property~\eqref{enum:local-c}, the second from Equation~\eqref{eq:app-local-bic-b}, and the last from Property~\eqref{enum:local-a}.
\end{proof}
Note that $1+\frac{1}{\blocal}\lceil \frac{\blocal}{\epsilon} \rceil -\frac{1}{\blocal} < 1 + \frac{1}{\epsilon}$.

\subsection{Discretized version of Algorithm~\ref{alg:localecc}}\label{app:local-pseudocode}
Algorithm~\ref{alg:localecc-discrete} is a discretized version of Algorithm~\ref{alg:localecc}. 
Note that the proof Lemma~\ref{thm:local-rt} is based on this discretized version.
\begin{algorithm}
\caption{Discretized primal-dual algorithm for \localecc}\label{alg:localecc-discrete}
\begin{algorithmic}
    \STATE $\ell_e \gets 0$ for all $e \in E$
    \STATE $L \gets \{e \in E \mid w_e > 0\}$
    \FOR{$v \in V$}
        \IF{$|\colorset(\delta(v)\cap L)| > b_v$}
            \STATE $\slack(c) \gets 0$ for all $c \in \colorset(\delta(v)\cap L)$
            \FOR{$e \in \delta(v)\cap L$}
                \STATE $\slack({c_e}) \gets \slack({c_e}) + (w_e - \ell_e)$
            \ENDFOR
            \STATE let $s^\star$ be the ($b_v+1$)-st largest value in the (multi)set $\{\slack(c)\}_{c\in\colorset(\delta(v)\cap L)}$
            \FOR{$c \in \colorset(\delta(v)\cap L)$}
                \STATE $\ell_e \gets \ell_e + \frac{\min\{\slack(c),s^\star\}}{\slack(c)}(w_e-\ell_e)$ for all $e \in \delta_c(v)\cap L$
                \IF{$\slack(c) \le s^\star$}
                    \STATE $L \gets L \setminus \delta_c(v)$
                \ENDIF
            \ENDFOR
            \ENDIF
        \STATE $\sigma(v) \gets \colorset(\delta(v)\cap L)$
    \ENDFOR
\end{algorithmic}
\end{algorithm}
\section{Technical details and proofs for Robust ECC deferred from Section~\ref{sec:robustglobal}}\label{app:robust}
\subsection{Primal LP}\label{app:robust-primal}
Following is the LP relaxation, where $z_v=1$ indicates that the node $v$ is removed from the hypergraph.
\begin{align*}
	\text{min } & \textstyle \sum_{e \in E}{w_e y_e} \\
	\text{s.t.\ } &\textstyle z_v+\sum_{c \in C}{x_{v,c}} \le 1, & \forall v \in V, \\
	& z_v+x_{v,c_e} + y_e \ge 1, & \forall e \in E, v \in e, \\
	& \textstyle \sum_{v\in V}{z_v} \le \brobust, & \\
	& x_{v,c} \ge 0, & \forall v \in V, c \in C, \\
	& y_e \ge 0, & \forall e \in E, \\
	& z_v \ge 0, & \forall v \in V.
\end{align*}
We note that the only difference between our LP and Crane et al.'s lies in the constraint $z_v + \sum_{e\in C} x_{v,c} \leq 1$. (The two LPs use opposite senses for the binary variable $x$, but this is not an inherent difference.) This difference turns out to be enough to reduce the integrality gap of our relaxation. See Theorem~\ref{thm:global-approx}.

\subsection{Proposed algorithm for Robust ECC}\label{app:robust-alg}
In Section~\ref{sec:robustglobal}, we sketched our algorithm for \robustecc. We present its pseudocode below.
\begin{algorithm}
    \caption{Proposed algorithm for \robustecc}\label{alg:robustecc}
\begin{algorithmic}
    \STATE $\mathbf{\alpha} \gets \mathbf{0};~\mathbf{\beta} \gets \mathbf{0};~\lambda \gets 0$
    \STATE $L \gets \{e \in E \mid w_e > 0\}$
    \STATE $R \gets \{v \in V \mid |\colorset(\delta(v)\cap L)|\ge 2\}$
    \WHILE{$|R| > \brobust$}
        \STATE increase $\lambda$ and $\alpha_v$ and $\beta_{e,v}$ for $v\in R$ and $e \in \delta(v) \cap L$ in a way that the increase rate of $\lambda$ and that of $\sum_{e \in \delta(v)\cap L}{\beta_{e,v}}-\alpha_v$ for each $v\in R$ are uniform and, for each $v \in R$, the increase rate of $\alpha_v$ and that of $\sum_{e \in \delta_c(v)\cap L}{\beta_{e,v}}$ for each $c \in \colorset(\delta(v)\cap L)$ are uniform, until there exists $e$ such that $\sum_{u \in e}{\beta_{e,u}} = w_e$
        \STATE \textbf{if} $\exists e\ \sum_{u \in e}{\beta_{e,u}} = w_e$ \textbf{then} remove all such edges from $L$
        \STATE \textbf{if} $\exists v\ |\colorset(\delta(v)\cap L)| \le 1$ \textbf{then} remove all such nodes from $R$
    \ENDWHILE
    \STATE remove $R$ from the hypergraph
    \FOR{$v \notin R$}
        \IF{$|\colorset(\delta(v)\cap L)|=1$}
		\STATE $\sigma(v) \gets c$ where $c \in \colorset(\delta(v)\cap L)$        
        \ELSE
		\STATE $\sigma(v) \gets c$ where $c$ is an arbitrary color
	    \ENDIF
    \ENDFOR
\end{algorithmic}
\end{algorithm}

\subsection{Proof of Theorem~\ref{thm:robust-approx}}\label{app:robust-thm}

We  have the following key lemma. Let us  prove only Property~\eqref{enum:rob-a},
since Properties~\eqref{enum:rob-b} and \eqref{enum:rob-c} can be seen from the same argument as the one for Lemma~\ref{lem:local-key}.
\begin{lemma} \label{lem:rob-key}
Algorithm~\ref{alg:robustecc} satisfies the following:
\begin{enumerate}[(a)]
\item \label{enum:rob-a} At any moment, $(\alpha, \beta, \lambda)$ is feasible to the dual LP.
\item \label{enum:rob-b} At any moment, for all $v \in V$, $\alpha_v \leq \frac{1}{2} \sum_{e \in \delta(v)} \beta_{e,v}$.
\item \label{enum:rob-c} At termination, every mistake $e$ under $\sigma$ is tight.
\end{enumerate}
\end{lemma}
\begin{proof}[Proof of Property~\eqref{enum:rob-a}]
    Recall the two properties of the algorithm. Observe that the first set of dual constraints remain feasible due to Property~\eqref{enum:robalg-b}; the second set of constraints are satisfied since the algorithm stops increasing  dual variables as soon as it discovers a tight edge; the third set of dual constraints  are kept feasible due to Property~\eqref{enum:robalg-a}.
\end{proof}

Let $\ALG$ be the total weight of mistakes in the output of Algorithm~\ref{alg:robustecc} and $\OPT$ be the weight of an optimal solution.
\begin{lemma}\label{thm:robust-rho}
We have $\ALG \leq (2\brobust+2) \cdot \OPT$.
\end{lemma}
\begin{proof}
Observe that, if $|R| \leq \brobust$ from the very beginning of the algorithm, the algorithm immediately terminates and incurs no weight. Let us thus assume  that $|R| > \brobust$ at the beginning.

Consider the timepoint when the algorithm terminates. Let $R_0$ denote the value of $R$ at termination. Let $R':=\{v \in V \mid \sum_{e \in \delta(v)}{\beta_{e,v}} - \alpha_v = \lambda\}$. We claim that $R_0 \subsetneq R'$ and $|R'|>\brobust$. (\emph{Proof.} Right before the algorithm terminates, it removes some nodes from  $R$. Consider the moment right before this removal. At this moment, every node $v$ in $R$ satisfies $\sum_{e \in \delta(v)}{\beta_{e,v}} - \alpha_v = \lambda$ and therefore is in $R'$. Note that $R$ contains more than $\brobust$ vertices at this moment, since otherwise the algorithm would have terminated earlier. Note that $R_0$ is the set resulting from the removal.)
Let $R''$ be any set such that $R_0 \subseteq R'' \subsetneq R'$ with $|R''| = \brobust$, and let $w$ denote an arbitrary node in $R' \setminus R''$. We can then bound $\OPT$ from below as follows:
\begin{align}
\OPT & \geq  \sum_{e \in E}{\sum_{v \in e}{\beta_{e,v}}}-\sum_{v\in V}{\alpha_v}-\lambda \brobust
 = \sum_{e \in E}{\sum_{v \in e}{\beta_{e,v}}}-\sum_{v\in V}{\alpha_v}-\sum_{v \in R''}{\Big(\sum_{e \in \delta(v)}{\beta_{e,v}}-\alpha_v\Big)} \nonumber \\
& = \sum_{v \in V \setminus R''}{\Big(\sum_{e \in \delta(v)}{\beta_{e,v}}-\alpha_v\Big)} \label{eq:rob-lb-1} \\
& \geq \frac{1}{2} \sum_{v \in V \setminus R''} \sum_{e \in \delta(v)} \beta_{e,v }, \label{eq:rob-lb-2}
\end{align}
where the first inequality is due to Property~\eqref{enum:rob-a} and the second inequality is due to Property~\eqref{enum:rob-b}. Moreover, since $w \in R' \setminus R''$ and $|R''| = \brobust$, we can find another lower bound on $\OPT$ from~\eqref{eq:rob-lb-1}:
\begin{align}
\OPT & \geq \sum_{v \in V \setminus R''}{\Big(\sum_{e \in \delta(v)}{\beta_{e,v}}-\alpha_v\Big)}
\ge \sum_{e \in \delta(w)}{\beta_{e,w}}-\alpha_w  
=\frac{1}{\brobust}\sum_{v \in R''}{\Big(\sum_{e \in \delta(v)}{\beta_{e,v}}-\alpha_v\Big)} \nonumber \\
& \geq \frac{1}{2\brobust} \sum_{v \in R''} \sum_{e \in \delta(v)} \beta_{e, v}, \label{eq:rob-lb-3}
\end{align}
where the equality follows from the fact that $w \in R'$ and the last inequality is again due to Property~\eqref{enum:rob-b}. Therefore, by Property~\eqref{enum:rob-c}, we have
\begin{align}
\ALG \leq \sum_{v \in V}\sum_{e \in \delta(v)} \beta_{e, v} 
 = \sum_{v \in R''}\sum_{e \in \delta(v)} \beta_{e, v} + \sum_{v \in V \setminus R''}\sum_{e \in \delta(v)} \beta_{e, v} 
  \leq (2\brobust + 2) \cdot \OPT, \label{eq:rob-bound}
\end{align}
where the last inequality follows from~\eqref{eq:rob-lb-2} and~\eqref{eq:rob-lb-3}.
Note that, if $\brobust=0$, \eqref{eq:rob-bound} immediately follows from~\eqref{eq:rob-lb-2} without~\eqref{eq:rob-lb-3}.
\end{proof}

Lemma~\ref{thm:robustglobal-rt} in Appendix~\ref{app:global-thm} shows that Algorithm~\ref{alg:robustecc} can be implemented to run in $O(|E|\sum_{v\in V}{d_v})$ time.

\robustApprox*
\begin{proof}
    Immediate from Lemmas~\ref{thm:robust-rho} and~\ref{thm:robustglobal-rt}.
\end{proof}

\subsection{Proof of Theorem~\ref{thm:robust-ig}}\label{app:robust-ig}
\robustIG*
\begin{proof}
    Consider a hypergraph $H=(V=\{v_1,\cdots,v_{\brobust+1}\},E=\{e_1, e_2\})$ where $e_1 = e_2 = V$, $w_{e_1}=w_{e_2}=1$, and $c_{e_1}\neq c_{e_2}$. Any integral solution incurs at least $1$ since at least one node should remain in the hypergraph and at least one edge cannot be satisfied. However, consider the solution given by $z_v=\frac{\brobust}{\brobust+1}$, $x_{v,c_{e_1}}=x_{v,c_{e_2}}=\frac{1}{2(\brobust+1)}$ for all $v\in V$ and $y_{e_1}=y_{e_2}=\frac{1}{2(\brobust+1)}$. This solution is feasible and the cost is $\frac{1}{\brobust+1}$.
\end{proof}

\subsection{Bicriteria algorithm for Robust ECC}\label{app:robust-bicriteria}
\begin{theorem}
    Suppose that $b\geq 1$. For any $\epsilon \in (0,2\brobust]$, there exists a $(2+\epsilon,1+\frac{1}{\brobust}\lceil \frac{2\brobust}{\epsilon} \rceil -\frac{1}{\brobust})$-approximation algorithm for \robustecc.
\end{theorem}
\begin{proof}
    Let $\tau:=\lceil \frac{2\brobust}{\epsilon} \rceil -1$. Consider the algorithm where the condition of the \textbf{while} loop in Algorithm~\ref{alg:robustecc} is replaced by $|R| > \brobust + \tau$. 
    It is clear that the number of removals is at most $\brobust+\tau = \brobust\cdot(1+\frac{1}{\brobust}\lceil \frac{2\brobust}{\epsilon} \rceil -\frac{1}{\brobust})$. 
    Note also that the modified algorithm  satisfies all the properties of Lemma~\ref{lem:rob-key}.
    
    We basically follow the proof of Lemma~\ref{thm:robust-rho}. Let $R_0$ be the set of nodes that are removed from the hypergraph $H$. Observe that $|R_0|\le \brobust+\tau$ from the construction. Let $\sigma$ be the color assignment of $V\setminus R_0$ output by the algorithm and let $E_m$ be the set of mistakes under $\sigma$. We have
\begin{equation}\label{eq:robustecc-bi-alg-cost}    
    \sum_{e\in E_m}{w_e} \leq \sum_{v \in V} \sum_{e \in \delta(v)} \beta_{e, v} \le 2 \sum_{v\in V}{\left(\sum_{e\in \delta(v)}{\beta_{e,v}-\alpha_v}\right)},
\end{equation}
where the first inequality follows from Property~\eqref{enum:rob-c} and the second from Property~\eqref{enum:rob-b}.
Let $R':=\{v \in V \mid \sum_{e \in \delta(v)}{\beta_{e,v}} - \alpha_v = \lambda\}$. Observe that $R_0 \subsetneq R'$ and $|R'|>\brobust+\tau$. Let $R''$ be a subset of $R'$ such that $R \subseteq R''$ with $|R''| = \brobust+\tau$. We then have
\begin{align}
    \OPT & \ge \sum_{e \in E}{\sum_{v \in e}{\beta_{e,v}}}-\sum_{v\in V}{\alpha_v}-\lambda \brobust \nonumber \\
    & = \sum_{v \in V}{\left(\sum_{e \in \delta(v)}{\beta_{e,v}}-\alpha_v\right)}-\frac{\brobust}{\brobust+\tau}\sum_{v \in R''}{\left(\sum_{e \in \delta(v)}{\beta_{e,v}}-\alpha_v\right)}  \label{eq:robustecc-bi-dual-1}\\
    & = \sum_{v \in V \setminus R''}{\left(\sum_{e \in \delta(v)}{\beta_{e,v}}-\alpha_v\right)}+\frac{\tau}{\brobust+\tau}\sum_{v \in R''}{\left(\sum_{e \in \delta(v)}{\beta_{e,v}}-\alpha_v\right)}.\label{eq:robustecc-bi-dual-2}
\end{align}
Let $w$ denote any node in $R' \setminus R''$. We give a lower bound on the first term of~(\ref{eq:robustecc-bi-dual-2}) as follows:
\begin{align}
    \sum_{v \in V \setminus R''}{\left(\sum_{e \in \delta(v)}{\beta_{e,v}}-\alpha_v\right)} \ge \sum_{e \in \delta(w)}{\beta_{e,w}}-\alpha_w =\frac{1}{\brobust+\tau}\sum_{v \in R''}{\left(\sum_{e \in \delta(v)}{\beta_{e,v}}-\alpha_v\right)}. \label{eq:robustecc-bi-dual-3}
\end{align}
Therefore by plugging~(\ref{eq:robustecc-bi-dual-3}) into~(\ref{eq:robustecc-bi-dual-2}), we have
\begin{align}
    \OPT \ge \frac{\tau+1}{\brobust+\tau}\sum_{v \in R''}{\left(\sum_{e \in \delta(v)}{\beta_{e,v}}-\alpha_v\right)}. \label{eq:robustecc-bi-dual-4}
\end{align}
We then have
\begin{align*}
    & \sum_{v\in V}\left({\sum_{e \in \delta(v)}{\beta_{e,v}} - \alpha_v} \right) \\
    & = \left[ \sum_{v\in V}\left({\sum_{e \in \delta(v)}{\beta_{e,v}} - \alpha_v} \right) -\frac{\brobust}{\brobust+\tau}\sum_{v \in R''}{\left(\sum_{e \in \delta(v)}{\beta_{e,v}}-\alpha_v\right)} \right] +\frac{\brobust}{\brobust+\tau}\sum_{v \in R''}{\left(\sum_{e \in \delta(v)}{\beta_{e,v}}-\alpha_v\right)} \\
    & \le \left(1+\frac{\brobust}{\tau+1}\right)\OPT \le \left(1+\frac{\epsilon}{2} \right)\OPT,
\end{align*}
where the first inequality follows from~\eqref{eq:robustecc-bi-dual-1} and~\eqref{eq:robustecc-bi-dual-4}.
Combining this inequality with~\eqref{eq:robustecc-bi-alg-cost} proves the theorem.
\end{proof}

Recall that Crane et al.~\cite{crane2024overlapping} gave an LP-rounding bicriteria $(2+\epsilon,2+\frac{4}{\epsilon})$-approximation algorithm for \robustecc for $\epsilon >0$. Note that this algorithm provides an better performance guarantee.

\subsection{Discretized version of Algorithm~\ref{alg:robustecc}}\label{app:robust-pseudocode}
Algorithm~\ref{alg:robustecc-discrete} is a discretized version of Algorithm~\ref{alg:robustecc}.
Note that the proof Lemma~\ref{thm:robustglobal-rt} is based on this discretized version.
\begin{algorithm}
\caption{Discretized primal-dual algorithm for \robustecc}\label{alg:robustecc-discrete}
\begin{algorithmic}
    \STATE $\ell_e \gets 0$ for all $e \in E$
    \STATE $L \gets \{e \in E \mid w_e > 0\}$
    \STATE $R \gets \{v \in V \mid |\colorset(\delta(v)\cap L)|\ge 2\}$
    \WHILE{$|R| > \brobust$}
        \STATE $\rate(e) \gets 0$ for all $e \in L$
        \FOR{$v \in R$}
            \FOR{$c \in \colorset(\delta(v)\cap L)$}
                \STATE $\rate(e) \gets \rate(e) + \frac{1}{|\colorset(\delta(v)\cap L)|-1}\cdot\frac{1}{|\delta_c(v)\cap L|}$ for all $e \in \delta_c(v)\cap L$
            \ENDFOR
        \ENDFOR
        \STATE $\inctime(e) \gets \frac{w_e-\ell_e}{\rate(e)}$ for all $e \in L$
        \STATE $t^\star \gets \min_{e\in L}{\inctime(e)}$
        \FOR{$e \in L$}
        \STATE $\ell_e \gets \ell_e + t^\star \cdot \rate(e)$
            \IF {$\ell_e = w_e$}
                \STATE $L \gets L \setminus \{e\}$
                \STATE remove all $v \in e$ from $R$ such that $|\colorset(\delta(v)\cap L)|\le 1$
            \ENDIF
        \ENDFOR
    \ENDWHILE
    \STATE remove $R$ from the hypergraph
    \FOR{$v \notin R$}
        \IF{$|\colorset(\delta(v)\cap L)|=1$}
		    \STATE $\sigma(v) \gets c$ where $c \in \colorset(\delta(v)\cap L)$        
        \ELSE
		    \STATE $\sigma(v) \gets c$ where $c$ is an arbitrary color
	    \ENDIF
    \ENDFOR
\end{algorithmic}
\end{algorithm}
\section{Technical details and proofs for Global ECC deferred from Section~\ref{sec:robustglobal}}\label{app:global}

\subsection{Primal and dual LP}
Following is the LP relaxation, where $z_v \in \mathbb{Z}_{\ge 0}$ indicates the number of additional colors budget assigned to $v$, i.e., $z_v + 1$ number of colors is assigned to $v$.
\begin{align*}
	\text{min } & \textstyle \sum_{e \in E}{w_e y_e} \\
	\text{s.t.\ } &\textstyle \sum_{c \in C}{x_{v,c}} \le z_v + 1, & \forall v \in V, \\
	& x_{v,c_e} + y_e \ge 1, & \forall e \in E, v \in e, \\
	& \textstyle \sum_{v\in V}{z_v} \le \bglobal, & \\
	& x_{v,c} \ge 0, & \forall v \in V, c \in C, \\
	& y_e \ge 0, & \forall e \in E, \\
	& z_v \ge 0, & \forall v \in V.
\end{align*}
Following is the dual of this LP.
\begin{align*}
        \text{max } &\textstyle \sum_{e \in E, v \in e}{\beta_{e,v}}-\sum_{v\in V}{\alpha_v}-\lambda \bglobal \\
        \text{s.t.\ } &\textstyle \sum_{e\in \delta_c(v)}{\beta_{e,v}} \le \alpha_v, & \forall v \in V, c \in C, \\
                     &\textstyle \sum_{v \in e}{\beta_{e,v}} \le w_e, & \forall e \in E, \\
                     &\alpha_v \le \lambda, & \forall v \in V, \\
                     & \alpha_v \ge 0, & \forall v \in V, \\
                     & \beta_{e,v} \ge 0, & \forall e \in E, v \in e, \\
                     & \lambda \ge 0.
\end{align*}

\subsection{Proposed algorithm for Global ECC}\label{app:global-alg}
Let us now present our algorithm. Similarly, we consider the problem where each edge has an associated weight $w_e$.

The algorithm shares many similarity with the algorithm for \robustecc.
It maintains a dual feasible solution $(\alpha, \beta, \lambda)$, starting from $(\mathbf{0}, \mathbf{0}, 0)$, the set $L$ of loose edges, and the set $R$ of nodes with at least two incident loose edges of distinct colors. It also simultaneously increases $\lambda$ and $(\alpha, \beta)$ associated with the nodes in $R$. Intuitively, $R$ is the set of nodes that we will assign (possibly) more than one color.

As the dual formulation differs, the way the algorithm increases the dual solution slightly varies. The following properties will hold:
\begin{enumerate}[(i)]
\item \label{enum:gloalg-a} $\lambda$ and $\alpha_v$ for each $v \in R$ increas at the same rate.
\item \label{enum:gloalg-b} For all $v\in R$, $\alpha_v$ and $\sum_{e \in \delta_c(v)\cap L}{\beta_{e,v}}$ for each $c \in \colorset(\delta(v) \cap L)$ increase at the same rate.
\end{enumerate}
Observe that these properties can be easily ensured. 

The algorithm increases the dual solution until $\sum_{v\in R}(|\colorset(\delta(v)\cap L)|-1)$ becomes at most $\bglobal$, and once the algorithm reaches this point, it assigns every node $v \in V$ the color in $\colorset(\delta(v) \cap L)$. It is clear that the returned node coloring is feasible. As before, if $\colorset(\delta(v) \cap L) = \emptyset$, an arbitrary color can be assigned without affecting the theoretical guarantee; in practical implementation, we could employ heuristics for marginal improvement.
See Algorithm~\ref{alg:globalecc} for a detailed pseudocode. The full discretized version of the algorithm is presented in Appendix~\ref{app:global-pseudocode}.

\begin{algorithm}
\caption{Proposed algorithm for \globalecc}\label{alg:globalecc}
\begin{algorithmic}
    \STATE $\mathbf{\alpha} \gets \mathbf{0};~\mathbf{\beta} \gets \mathbf{0};~\lambda \gets 0$
    \STATE $L \gets \{e \in E \mid w_e > 0\}$
    \STATE $R \gets \{v \in V \mid |\colorset(\delta(v)\cap L)|\ge 2\}$
    \WHILE{$\sum_{v\in R}(|\colorset(\delta(v)\cap L)|-1) > \bglobal$}
        \STATE increase $\lambda$ and $\alpha_v$ and $\beta_{e,v}$ for $v\in R$ and $e \in \delta(v) \cap L$ in a way that the increase rate of $\lambda$ and that of $\alpha_v$ for each $v\in R$ are uniform and, for each $v \in R$, the increase rate of $\alpha_v$ and that of $\sum_{e \in \delta_c(v)\cap L}{\beta_{e,v}}$ for each $c \in \colorset(\delta(v)\cap L)$ are uniform, until there exists $e$ such that $\sum_{u \in e}{\beta_{e,u}} = w_e$
        \STATE \textbf{if} $\exists e\ \sum_{u \in e}{\beta_{e,u}} = w_e$ \textbf{then} remove all such edges from $L$
        \STATE \textbf{if} $\exists v\ |\colorset(\delta(v)\cap L)| \le 1$ \textbf{then} remove all such nodes from $R$
    \ENDWHILE
    
    \FOR{$v \in V$}
        \IF{$|\colorset(\delta(v)\cap L)|\ge 1$}
		\STATE $\sigma(v) \gets \colorset(\delta(v)\cap L)$        
        \ELSE
		\STATE $\sigma(v) \gets \{c\}$ where $c$ is an arbitrary color
	    \ENDIF
    \ENDFOR
\end{algorithmic}
\end{algorithm}

\subsection{Proof of Theorem~\ref{thm:global-approx}}\label{app:global-thm}
We have the following key lemma. Let us prove only Property~\eqref{enum:glo-a}, since Properties~\eqref{enum:glo-b} through \eqref{enum:glo-d} can be seen from the same argument as the one for Lemma~\ref{lem:local-key}.
\begin{lemma} \label{lem:glo-key}
Algorithm~\ref{alg:globalecc} satisfies the following:
\begin{enumerate}[(a)]
\item \label{enum:glo-a} At any moment, $(\alpha, \beta, \lambda)$ is feasible to the dual LP.
\item \label{enum:glo-b} At any moment, for all $v \in V$, $\alpha_v \leq \frac{1}{2} \sum_{e \in \delta(v)} \beta_{e,v}$.
\item \label{enum:glo-c} At any moment, for all $v \in R$, $\alpha_v \leq \frac{1}{|\colorset(\delta(v)\cap L)|} \sum_{e \in \delta(v)} \beta_{e,v}$.
\item \label{enum:glo-d} At termination, every mistake $e$ under $\sigma$ is tight.
\end{enumerate}
\end{lemma}
\begin{proof}[Proof of Property~\eqref{enum:glo-a}]
    Recall the two properties of the algorithm. Observe that the first set of dual constraints remain feasible due to Property~\eqref{enum:gloalg-b}; the second set of constraints are satisfied since the algorithm stops increasing dual variables as soon as it discovers a tight edge; the third set of dual constraints are kept feasible due to Property~\eqref{enum:gloalg-a}.
\end{proof}

Given $L \subseteq E$, let $\localslack_L(v):=|\colorset(\delta(v)\cap L)|$ for $v \in V$.
Let $\budget_L(S):=\sum_{v\in R}(\localslack_L(v)-1)$ for $S\subseteq V$.
Let $\ALG$ be the total weight of mistakes in the output of Algorithm~\ref{alg:globalecc} and $\OPT$ be the weight of an optimal solution.
\begin{lemma}\label{thm:global-rho}
We have $\ALG \leq (2\bglobal+2) \cdot \OPT$.
\end{lemma}
\begin{proof}
Observe that, if $\budget_L(R) \leq \bglobal$ from the very beginning of the algorithm, the algorithm immediately terminates and incurs no weight. Let us thus assume that $\budget_L(R) > \bglobal$ at the beginning.

Consider the last iteration of \textbf{while} loop of the algorithm. 
In this iteration, the algorithm removes some edges from $L$ (and possibly removes some vertices).
Consider the moment right before the removal of edges. Let $R'$ and $L'$, respectively, denote the value of $R$ and $L$ at this moment.
Let $b' := \budget_{L'}(R')$.
Note that $b' > \bglobal$, since otherwise the algorithm would have terminated earlier.
Moreover, at this moment--or at the termination--we have $\alpha_v = \lambda$, for every $v \in R'$.

We then bound $\OPT$ from below as follows:
\begin{align}
\OPT & \geq \sum_{e \in E}{\sum_{v \in e}{\beta_{e,v}}}-\sum_{v\in V}{\alpha_v}-\lambda \bglobal
    = \sum_{v \in V}\left( \sum_{e \in \delta(v)}\beta_{e,v} -\alpha_v \right) - \frac{\bglobal}{b'}\sum_{v\in R'}(\localslack_{L'}(v)-1)\alpha_v \nonumber \\
    & = \sum_{v \in V \setminus R'}\left( \sum_{e \in \delta(v)}\beta_{e,v} -\alpha_v \right) + \sum_{v \in R'}\left( \sum_{e \in \delta(v)}\beta_{e,v} - \left(1 + \frac{\bglobal}{b'}(\localslack_{L'}(v)-1)\right)\alpha_v \right) \nonumber \\
    & \ge \sum_{v \in V \setminus R'}\left( \frac{1}{2}\sum_{e \in \delta(v)}\beta_{e,v} \right) + 
    \sum_{v \in R'}\left( \left(1- \frac{1}{\localslack_{L'}(v)}\left(1 + \frac{\bglobal}{b'}(\localslack_{L'}(v)-1)\right)\right)\sum_{e \in \delta(v)}\beta_{e,v} \right) \nonumber \\
    & = \sum_{v \in V \setminus R'}\left( \frac{1}{2}\sum_{e \in \delta(v)}\beta_{e,v} \right) + 
    \sum_{v \in R'}\left( \left(1- \frac{1}{\localslack_{L'}(v)}\right)\left(1 - \frac{\bglobal}{b'}\right)\sum_{e \in \delta(v)}\beta_{e,v} \right) \nonumber \\
    & \ge \sum_{v \in V \setminus R'}\left( \frac{1}{2}\sum_{e \in \delta(v)}\beta_{e,v} \right) + 
    \sum_{v \in R'}\left( \frac{1}{2}\left(1 - \frac{\bglobal}{b'}\right)\sum_{e \in \delta(v)}\beta_{e,v} \right) \nonumber \\
    & \ge \frac{1}{2}\left(1 - \frac{\bglobal}{b'}\right)\sum_{v \in V}\sum_{e \in \delta(v)}\beta_{e,v} \label{eq:glo-lb}
\end{align}
where the first inequality is due to Property~\eqref{enum:rob-a}, the first equality comes from $\lambda=\alpha_v$ for every $v \in R'$, the second inequality is due to Property~\eqref{enum:rob-b} and \eqref{enum:rob-c}, and the second to last inequality comes from $\localslack_{L'}(v)\ge 2$ for every $v \in R'$.
Therefore, by Property~\eqref{enum:glo-d}, we have
\begin{align}
\ALG \leq \sum_{v \in V}\sum_{e \in \delta(v)} \beta_{e, v} \leq (2\bglobal + 2) \cdot \OPT, \nonumber
\end{align}
where the last inequality comes from $b'\ge \bglobal+1$, which implies $1 - \frac{\bglobal}{b'} \ge 1 - \frac{\bglobal}{\bglobal+1} = \frac{1}{\bglobal+1}$.
\end{proof}

\begin{lemma}\label{thm:robustglobal-rt}
Both Algorithm~\ref{alg:robustecc} and Algorithm~\ref{alg:globalecc} can be implemented to run in $O(|E|\sum_{v\in V}{d_v})$ time.
\end{lemma}
\begin{proof}
    It suffices to show that we can decide in $O(\sum_{v \in V} d_v)$ time which edge becomes tight, as well as the increment of the dual variables.
    By iterating each node $v$ and its incident edges $\delta(v)$, we can compute the increase rates of $\alpha_v$ and  $\{ \beta_{e, v} \}_{e \in \delta(v)}$.
    From this, we can obtain the increase rate of the ``level'' $\ell_e:=\sum_{u \in e}{\beta_{e,u}}$ for each $e \in L$. Let $\rate(e)$ denote this increase rate.
    As $e \in L$ will become tight in $\inctime(e):=\frac{w_e - \sum_{u \in e}{\beta_{e,u}}}{\rate(e)}$ time, we can determine the edge that will become tight the earliest. We can also compute the increment of the dual variables accordingly.
\end{proof}

\globalApprox*
\begin{proof}
    Immediate from Lemmas~\ref{thm:global-rho} and~\ref{thm:robustglobal-rt}.
\end{proof}
We note that our framework gives a true (non-bicriteria) approximation algorithm. This shows that the LP, which is equivalent to that of Crane et al.~\citep{crane2024overlapping}, has an integrality gap of $O(\bglobal)$. 

\subsection{Proof of Theorem~\ref{thm:global-ig}}\label{app:global-ig}
\globalIG*
\begin{proof}
    We construct an instance similar to the one used in the proof of Theorem~\ref{app:robust-ig}.
    Consider a hypergraph $H=(V=\{v_1,\cdots,v_{\bglobal+1}\},E=\{e_1, e_2\})$ where $e_1 = e_2 = V$, $w_{e_1}=w_{e_2}=1$, and $c_{e_1}\neq c_{e_2}$. Any integral solution incurs at least $1$ since at least one node is assigned one color and at least one edge cannot be satisfied. However, consider the solution given by $z_v=\frac{\bglobal}{\bglobal+1}$, $x_{v,c_{e_1}}=x_{v,c_{e_2}}=\frac{2\bglobal+1}{2(\bglobal+1)}$ for all $v\in V$ and $y_{e_1}=y_{e_2}=\frac{1}{2(\bglobal+1)}$. This solution is feasible and the cost is $\frac{1}{\bglobal+1}$.
\end{proof}

\subsection{Bicriteria algorithm for \globalecc}\label{app:global-bicriteria}
\begin{theorem}
    Suppose that $b\geq 1$. For any $\epsilon \in (0,2\bglobal]$, there exists a $(2+\epsilon,1+\frac{1}{\bglobal}\lceil \frac{2\bglobal}{\epsilon} \rceil -\frac{1}{\bglobal})$-approximation algorithm for \globalecc.
\end{theorem}
\begin{proof}
    Let $\tau:=\lceil \frac{2\bglobal}{\epsilon} \rceil -1$. Consider the algorithm where the condition of the \textbf{while} loop in Algorithm~\ref{alg:globalecc} is replaced by $\sum_{v\in R}(|\colorset(\delta(v)\cap L)|-1) > \bglobal + \tau$. 
    It is clear that
    \[
    \sum_{v\in V}\max\{|\sigma(v)|-1,0\} \le \bglobal+\tau = \bglobal\cdot\left(1+\frac{1}{b}\left\lceil \frac{2\bglobal}{\epsilon} \right\rceil -\frac{1}{\bglobal}\right).
    \]
    Note also that the modified algorithm satisfies all the properties of Lemma~\ref{lem:glo-key}.
    
    We basically follow the proof of Lemma~\ref{thm:global-rho}. With the same definition of $R'$ and $L'$, we have $b':=\budget_{L'}(R') > \bglobal+\tau$. 
    Note that $1 - \frac{\bglobal}{b'} \ge 1 - \frac{\bglobal}{\bglobal+\tau+1} = \frac{\tau+1}{\bglobal+\tau+1}$.
    Together with equation~\eqref{eq:glo-lb},
    \begin{align*}
        \ALG \leq \sum_{v \in V}\sum_{e \in \delta(v)} \beta_{e, v} \leq (2+\frac{2\bglobal}{\tau+1}) \cdot \OPT \leq (2+\epsilon) \cdot \OPT,
    \end{align*}
    where the last inequality comes from $\tau \ge \frac{2\bglobal}{\epsilon}-1$.
\end{proof}

Recall that Crane et al.~\cite{crane2024overlapping} gave an LP-rounding bicriteria $(\bglobal+3+\epsilon,1+\frac{\bglobal+2}{\epsilon})$-approximation algorithm for \globalecc for $\epsilon >0$. Both approximation factor and violation factor of Crane et al.'s algorithm are linear in $\bglobal$. They raised an open question whether we can give a bicriteria approximation algorithm for \globalecc with both factor being constant (or give a hardness result). Observe $1+\frac{1}{b}\lceil \frac{2\bglobal}{\epsilon} \rceil -\frac{1}{\bglobal} < 1+\frac{2}{\epsilon}$. Our bicriteria algorithm satisfies the condition, answering the open question of Crane et al.~\cite{crane2024overlapping}.

\subsection{Discretized version of Algorithm~\ref{alg:globalecc}}\label{app:global-pseudocode}
Algorithm~\ref{alg:globalecc-discrete} is a discretized version of Algorithm~\ref{alg:globalecc}.
Note that the proof Lemma~\ref{thm:robustglobal-rt} is based on this discretized version.
\begin{algorithm}
\caption{Discretized primal-dual algorithm for \globalecc}\label{alg:globalecc-discrete}
\begin{algorithmic}
    \STATE $\ell_e \gets 0$ for all $e \in E$
    \STATE $L \gets \{e \in E \mid w_e > 0\}$
    \STATE $R \gets \{v \in V \mid |\colorset(\delta(v)\cap L)|\ge 2\}$
    \WHILE{$\sum_{v\in R}(|\colorset(\delta(v)\cap L)|-1) > \bglobal$}
        \STATE $\rate(e) \gets 0$ for all $e \in L$
        \FOR{$v \in R$}
            \FOR{$c \in \colorset(\delta(v)\cap L)$}
                \STATE $\rate(e) \gets \rate(e) + \frac{1}{|\delta_c(v)\cap L|}$ for all $e \in \delta_c(v)\cap L$
            \ENDFOR
        \ENDFOR
        \STATE $\inctime(e) \gets \frac{w_e-\ell_e}{\rate(e)}$ for all $e \in L$
        \STATE $t^\star \gets \min_{e\in L}{\inctime(e)}$
        \FOR{$e \in L$}
        \STATE $\ell_e \gets \ell_e + t^\star \cdot \rate(e)$
            \IF {$\ell_e = w_e$}
                \STATE $L \gets L \setminus \{e\}$
                \STATE remove all $v \in e$ from $R$ such that $|\colorset(\delta(v)\cap L)|\le 1$
            \ENDIF
        \ENDFOR
    \ENDWHILE
    \STATE $\sigma(v) \gets \colorset(\delta(v)\cap L)$ for all $v \in V$
    \end{algorithmic}
\end{algorithm}
\section{Dataset description}\label{app:data}
The benchmark data of Crane et al.~\cite{crane2024overlapping} contains six datasets. \cooking~\cite{whats-cooking,amburg2020clustering}, described in Section~\ref{sec:intro}, is a hypergraph whose nodes correspond to food ingredients, edges represent recipes, and edge colors indicate cuisines.
\brain~\cite{rubinov2010complex, crossley2013cognitive} contains the relation between brain regions: there are two types of relations, coactivation and connectivity, encoded by colors. The nodes corresponds to brain regions, and colored edges represent relations between them.
In \magten~\cite{sinha2015overview,amburg2020clustering}, each node corresponds to a researcher, and an edge indicates the author set of a published paper. Its color represents the publication venue (e.g., NeurIPS). 
\dawn~\cite{dawn2011,amburg2020clustering} is a dataset on the relation between drug use and emergency room (ER) visit disposition such as ``discharged'', ``surgery'', and ``transferred''. Each node corresponds to a drug, an edge corresponds to the combination of drugs taken by an ER patient, and colors represent the visit disposition.
In \walmart~\cite{walmart2015,amburg2020clustering}, each node represents a product, an edge indicates a set of products purchased together, and colors are ``trip type'' labels determined by Walmart.
Lastly in \trivago~\cite{knees2019recsys,chodrow2021hypergraph}, nodes are vacation rental properties and an edge represents the set of rental properties clicked during a single browsing session of a single user. Colors correspond to the countries where the browsing sessions happen.

\section{Tables and figures deferred from Section~\ref{sec:exp-robustglobal}}\label{app:exp-graphs}

\begin{table}[h!]
	\caption{Average running times of each dataset (in seconds): \robustecc. Values in parentheses are averages excluding trivial instances.}\label{tab:exp-robust-rt}
	\centering
	\begin{tabular}{cccc}
		\toprule
		& Proposed & ~~Greedy~~ & LP-rounding \\
		\midrule
		\brain & \phantom{00}1.345~~~\phantom{00}(1.345) & 0.005~~~(0.005) & \phantom{000}2.007~~~\phantom{000}(2.007) \\
		\magten & \phantom{0}11.056~~~\phantom{0}(15.303) & 0.666~~~(0.588) & \phantom{00}15.871~~~\phantom{00}(17.660) \\
		\cooking & \phantom{0}42.571~~~\phantom{0}(42.571) & 0.118~~~(0.118) & \phantom{0}220.107~~~\phantom{0}(220.107) \\
		\dawn & \phantom{0}39.905~~~\phantom{0}(39.905) & 0.048~~~(0.048) & \phantom{00}16.464~~~\phantom{00}(16.464)\\
		\walmart & 243.881~~~(243.881) & 1.995~~~(1.995) & 3766.539~~~(3766.539) \\
		\trivago & 195.337~~~(227.799) & 5.210~~~(5.144) & \phantom{0}705.323~~~\phantom{0}(709.378) \\
		\bottomrule
	\end{tabular}
\end{table}

\begin{figure}[h!]
	\centering
	\includegraphics[width=\columnwidth]{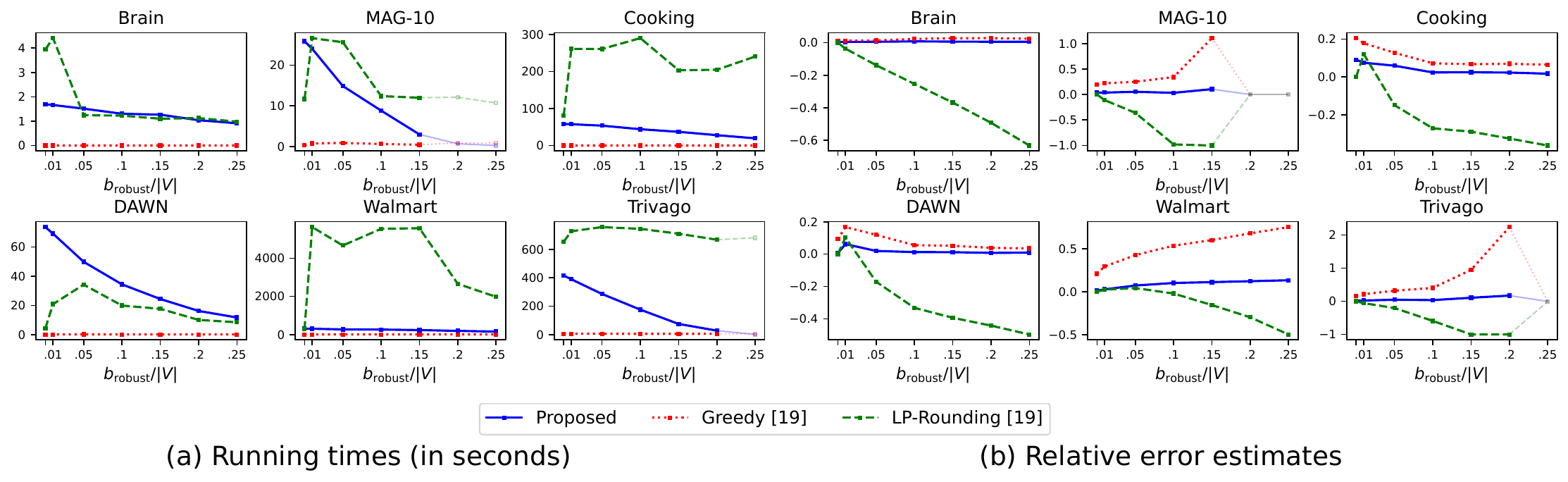}
	\caption{
		(a) Running times (in seconds) and (b) relative error estimates of the \robustecc algorithms. Empty square markers denote trivial instances.
	}\label{fig:exp-robust}
\end{figure}

\begin{table}[h!]
	\caption{Average running times of each dataset (in seconds): \globalecc. Values in parentheses are averages excluding trivial instances.}\label{tab:exp-global-rt}
	
	\centering
	\begin{tabular}{cccc}
		\toprule
		& Proposed & Greedy & LP-rounding \\
		\midrule
		\brain & \phantom{00}0.415~~~\phantom{00}(0.885) & 0.008~~~(0.012) & \phantom{00}1.849~~~\phantom{00}(3.381) \\
		\magten & \phantom{00}3.143~~~\phantom{0}(12.541) & 0.787~~~(0.625)  & \phantom{0}12.194~~~\phantom{0}(14.418) \\
		\cooking & \phantom{0}26.755~~~\phantom{0}(31.609)  & 0.171~~~(0.177) & \phantom{0}75.953~~~\phantom{0}(88.883) \\
		\dawn & \phantom{0}20.345~~~\phantom{0}(26.443) & 0.054~~~(0.052) & \phantom{00}7.944~~~\phantom{00}(9.263) \\
		\walmart & 117.046~~~(189.895) & 2.310~~~(2.247) & 511.011~~~(754.344) \\
		\trivago & \phantom{0}50.162~~~(107.449) & 7.234~~~(6.795) & 684.776~~~(662.595) \\
		\bottomrule
	\end{tabular}
\end{table}

\begin{figure}[h!]
	\centering\includegraphics[width=\columnwidth]{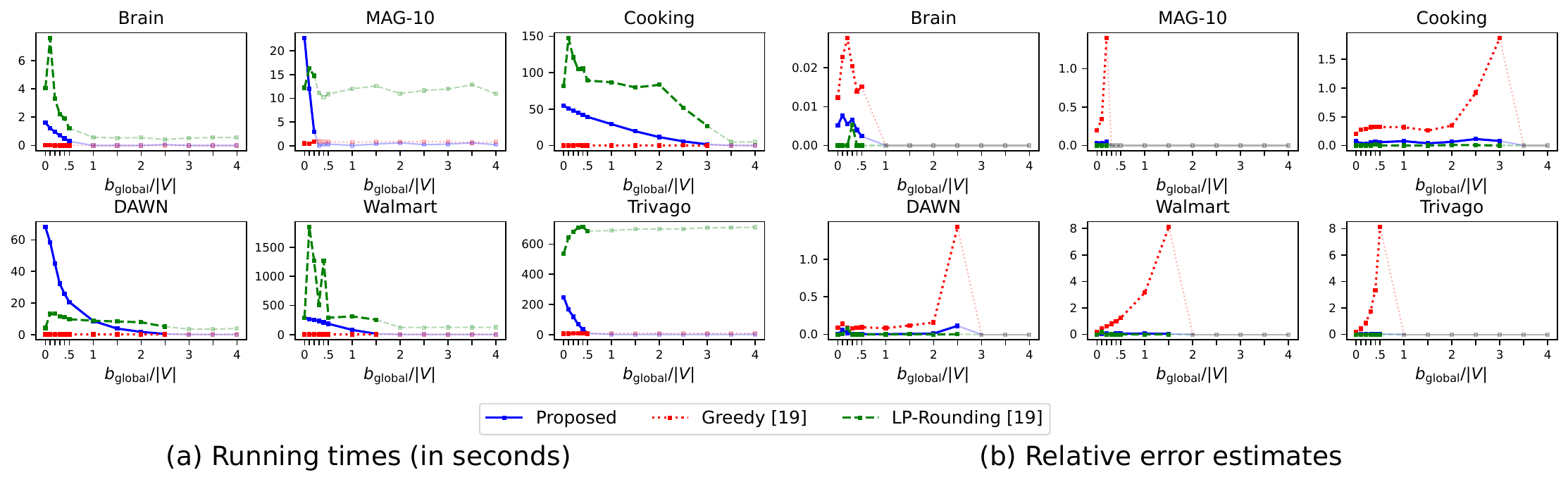}
	\caption{
		(a) Running times (in seconds) and (b) relative error estimates of the \globalecc algorithms. Empty square markers denote trivial instances.
	}\label{fig:exp-global}
\end{figure}

\end{document}